\documentclass[10pt]{article}

\usepackage[margin=1in]{geometry}
\usepackage{amsmath,amssymb,amsfonts,amsthm,mathtools}
\usepackage{enumitem}
\usepackage{array}
\usepackage{float}
\usepackage{placeins}
\usepackage{algpseudocode}
\usepackage{xcolor}
\usepackage{graphicx}
\usepackage{subfig}
\usepackage{capt-of}
\usepackage{hyperref}


\floatstyle{ruled}
\newfloat{algorithm}{t}{loa}
\floatname{algorithm}{Algorithm}

\hypersetup{
  colorlinks=true,
  linkcolor=blue,
  citecolor=blue,
  urlcolor=blue
}

\newtheorem{theorem}{Theorem}[section]
\newtheorem{corollary}[theorem]{Corollary}
\newtheorem{lemma}[theorem]{Lemma}
\newtheorem{proposition}[theorem]{Proposition}
\newtheorem{assumption}{Assumption}[section]
\newtheorem{condition}{Condition}[section]
\newtheorem{remark}{Remark}[section]
\newtheorem{definition}{Definition}[section]

\DeclareMathOperator{\msign}{msign}
\DeclareMathOperator{\rank}{rank}
\DeclareMathOperator{\Tr}{Tr}

\newcommand{\R}{\mathbb{R}}
\newcommand{\E}{\mathbb{E}}
\newcommand{\Prb}{\mathbb{P}}
\newcommand{\cD}{\mathcal{D}}
\newcommand{\cI}{\mathcal{I}}
\newcommand{\Orth}{\mathrm{Orth}}
\newcommand{\wh}{\widehat}
\newcommand{\bX}{\mathbf X}
\newcommand{\bY}{\mathbf Y}
\newcommand{\bM}{\mathbf M}
\newcommand{\bO}{\mathbf O}
\newcommand{\bG}{\mathbf G}
\newcommand{\bU}{\mathbf U}
\newcommand{\bV}{\mathbf V}
\newcommand{\bSigma}{\boldsymbol{\Sigma}}
\newcommand{\ip}[2]{\langle #1,#2\rangle}
\newcommand{\norm}[1]{\left\lVert #1\right\rVert}
\newcolumntype{P}[1]{>{\raggedright\arraybackslash}p{#1}}

\title{A Note on Stability for Orthogonalized Matrix Momentum with Client Sampling}
\author{
Da Chang$^{1}$,
Qiankun Shi$^{2}$,
Lvgang Zhang$^{3}$, 
Yu Li$^{4}$, 
Ruijie Zhang$^{1}$ \\[2mm]
{\small $^{1}$University of Chinese Academy of Sciences} 
{\small $^{2}$Sun Yat-sen University}\\
{\small $^{3}$Southern University of Science and Technology} 
{\small $^{4}$George Washington University}
}
\date{}

\begin{document}
\maketitle

\begin{abstract}
We study finite-sample generalization for a client-sampled distributed optimization scheme with matrix-valued parameters and orthogonalized momentum updates. The central quantity is the gap between the population and empirical objectives at the returned model when only a subset of clients participates in each round. Under independent heterogeneous client data, unequal local sample counts, and fixed aggregation weights, we derive a finite-round upper-tail guarantee from a coupled-neighbor stability recursion and a weighted concentration step. The bound keeps the client-selection counts through the amplification factor \(Y_i(\mathcal C)\); in the uniform full-participation full-batch regime, it yields \(\widetilde{\mathcal O}(n^{-1}+n^{-1/2})\) scaling whenever the horizon-dependent amplification terms are controlled. The matrix-orthogonalization rule is required to be Lipschitz along paired trajectories, a condition satisfied by regularized polar-type maps and normalized finite-step Newton--Schulz orthogonalizers. For the unregularized matrix sign, the same argument requires coupled spectral separation, whereas Gaussian smoothing gives a finite-round smoothed variant. A one-dimensional counterexample shows why a gap, smoothing, or regularity condition is necessary.
\end{abstract}

\section{Introduction}

Existing FedMuon analyses focus on how Muon-style matrix orthogonalization changes optimization in federated learning. Their convergence statements typically control averaged stationarity measures or related high-probability analogues for momentum SGD-type methods. We study a complementary finite-sample question: how to control the population-risk gap \(F(\bX_{\rm out})-\wh F_S(\bX_{\rm out})\) at the data-dependent FedMuon output.

The bound is derived by algorithmic stability, using the stability-to-generalization framework of Bousquet and Elisseeff~\cite{bousquet2002stability} and its later use for stochastic gradient methods~\cite{hardt2016train}. The high-probability conversion uses a weighted sharp-stability argument, with a weighted bounded-differences bound retained as a fallback.

The proof uses two facts about Muon~\cite{jordan2024muon,pethick2025training,chang2025convergence}. For \(\bM=\bU\bSigma\bV^\top\), the matrix-sign direction \(\msign(\bM)=\bU\bV^\top\) satisfies \(\norm{\msign(\bM)}_F\le\sqrt q\), where \(q:=\min\{d_1,d_2\}\). FedMuon convergence proofs use this bounded update property to control consensus and client drift~\cite{zhang2025provable,liu2025fedmuon}. However, the exact matrix-sign map is nonlinear and not globally Lipschitz, consistent with classical sensitivity results for polar decompositions~\cite{higham1986computing,li1995new}. A stability proof for a risk-gap bound therefore requires a smoothed or regularized orthogonalization map, a Newton--Schulz map with an explicit Lipschitz bound~\cite{schulz1933iterative,bernstein2024old,amsel2025polar}, or a neighboring-trajectory spectral-gap condition.

The main result is a conditional finite-horizon stability bound for FedMuon with Lipschitz-controlled orthogonalization. It gives a schedule-conditional one-sided high-probability bound for \(F(\bX_{\rm out})-\wh F_S(\bX_{\rm out})\) in terms of client participation, replacement sensitivity, and the Lipschitz constant of the orthogonalization map. The theorem separates the stability recursion from the verification of the orthogonalization condition: smoothed polar maps and normalized finite-step Newton--Schulz maps give global Lipschitz instantiations, while exact matrix sign is handled either through a deterministic neighboring-tube spectral gap or by Gaussian smoothing. A scalar obstruction shows that the deterministic exact-sign gap requirement is substantive.

\medskip\noindent\textbf{Contributions.}
The paper makes the following contributions.
\begin{itemize}[leftmargin=*]
\item We derive a schedule-conditional neighboring-dataset stability recursion for FedMuon with partial client participation, heterogeneous local sample sizes, and weighted aggregation. The recursion keeps the client-level exposure factor \(Y_i(\mathcal C)\) explicit.
\item We convert the stability recursion into a one-sided high-probability bound for \(F(\bX_{\rm out})-\wh F_S(\bX_{\rm out})\) by using a weighted sharp-stability argument, with a weighted bounded-differences branch retained for unbalanced client weights.
\item We separate the FedMuon stability recursion from the regularity of the orthogonalization map. Smoothed polar maps and normalized finite-step Newton--Schulz maps satisfy the required Lipschitz condition globally, while exact matrix sign requires either a neighboring-tube spectral gap or Gaussian smoothing; the scalar obstruction shows that this regularity condition cannot be omitted.
\end{itemize}

\section{Related Work}
\label{sec:main-related-work}
FedMuon analyses study Muon-style matrix orthogonalization in federated optimization, including periodic averaging, partial participation, momentum aggregation, local--global alignment, and bias-corrected LMO variants~\cite{zhang2025provable,takezawa2025fedmuon,liu2025fedmuon}. These works provide optimization and stationarity guarantees, whereas we study the finite-sample risk gap of the FedMuon output.

Muon can be viewed as a matrix-sign or polar-factor update, closely related to norm-constrained linear minimization oracles and update-norm geometry~\cite{jordan2024muon,pethick2025training,bernstein2024old}. Since exact polar factors are not globally Lipschitz near rank changes~\cite{higham1986computing}, the stability proof requires a local Lipschitz condition, a smoothed polar map, a finite-step Newton--Schulz map, or a spectral-gap event~\cite{schulz1933iterative,amsel2025polar}.

MiMuon studies Muon generalization through algorithmic stability and proposes a mixed Muon and momentum SGD optimizer to reduce sensitivity to exact orthogonalization~\cite{huang2026mimuon}. Our analysis applies the same stability perspective to a federated algorithm with partial participation, heterogeneous sample sizes and weights, and schedule-dependent amplification. It also combines a one-sided high-probability risk-gap conversion with Lipschitz control of orthogonalization. The resulting theorem concerns finite-sample generalization, with unregularized exact matrix sign handled under a neighboring trajectory condition.

The proof follows the algorithmic-stability route from neighboring datasets to generalization~\cite{bousquet2002stability,hardt2016train}. For the high-probability step, we use sharp-stability techniques~\cite{feldman2019high,bousquet2020sharper}: the main bound contains a weighted analogue of the logarithmic stability overhead, while a McDiarmid-style weighted bounded-differences inequality~\cite{mcdiarmid1989method} is retained as a fallback for highly unbalanced client weights.

\section{Setup}
\label{sec:setup}

There are \(N\) clients. Client \(i\) has distribution \(\cD_i\) over a common sample space \(\mathcal Z\) and an independent sample set \(S_i=\{z_{i,1},\ldots,z_{i,n_i}\}\), with \(z_{i,j}\sim\cD_i\). Let
\[
n:=\sum_{i=1}^N n_i,
\qquad
\cI:=\{(i,j):i\in[N],\,j\in[n_i]\}.
\]
For a matrix parameter \(\bX\in\R^{d_1\times d_2}\), let the per-example loss be \(\ell:\R^{d_1\times d_2}\times \mathcal Z\to\R_+\). Fix client weights \(p_i>0\), \(\sum_{i=1}^Np_i=1\); the common choices are \(p_i=n_i/n\) and \(p_i=1/N\).

Define
\[
F_i(\bX):=\E_{z\sim\cD_i}[\ell(\bX,z)],
\qquad
\wh F_i(\bX):=\frac1{n_i}\sum_{j=1}^{n_i}\ell(\bX,z_{i,j}).
\]
The weighted global risks are
\[
F(\bX):=\sum_{i=1}^Np_iF_i(\bX),
\qquad
\wh F_S(\bX):=\sum_{i=1}^Np_i\wh F_i(\bX)
=
\sum_{i=1}^N\sum_{j=1}^{n_i}w_{i,j}\ell(\bX,z_{i,j}),
\quad
w_{i,j}:=\frac{p_i}{n_i}.
\]
The effective sample size associated with the weights is
\[
n_{\rm eff}:=
\frac{1}{\sum_{i=1}^N p_i^2/n_i}.
\]
If \(p_i=n_i/n\), then \(n_{\rm eff}=n\). If \(p_i=1/N\), then \(n_{\rm eff}=N^2/(\sum_{i=1}^N 1/n_i)\).

The theorem allows arbitrary fixed client distributions \(\cD_i\), provided that samples are independent across and within clients and the loss, replacement-sensitivity, and orthogonalization conditions stated below hold. Dirichlet non-IID sampling is represented by the choice of different \(\cD_i\)'s; no bounded-heterogeneity parameter such as \(\frac1N\sum_{i=1}^N\norm{\nabla F_i(\bX)-\nabla F(\bX)}_F^2\le \sigma_g^2\) enters the stability-to-generalization argument.

\subsection{Algorithm and Orthogonalization Maps}
\label{sec:algorithm}

Let \(q:=\min\{d_1,d_2\}\). For a matrix \(\bM\in\R^{d_1\times d_2}\) with compact SVD \(\bM=\bU\bSigma\bV^\top\), define the exact Muon matrix sign~\cite{jordan2024muon,pethick2025training} by \(\msign(\bM):=\bU\bV^\top\) and \(\msign(0):=0\). Then \(\norm{\msign(\bM)}_F=\sqrt{\rank(\bM)}\le\sqrt q\) and \(\ip{\bM}{\msign(\bM)}=\norm{\bM}_*\).
A Newton--Schulz implementation uses \(\Orth\) as an approximate orthogonalization map~\cite{schulz1933iterative,bernstein2024old,amsel2025polar}. The theorem analyzes a generic orthogonalization rule satisfying the Lipschitz condition in Condition~\ref{ass:polar}. This condition is global for smoothed polar maps and, by Proposition~\ref{prop:ns-finite-step-main}, for \(\epsilon_{\rm ns}\)-normalized finite-step Newton--Schulz maps with fixed coefficients. Deterministic exact matrix sign is covered under a joint full-rank spectral-gap tube containing all coupled neighboring momenta. The Gaussian-smoothed exact-sign variant in Corollary~\ref{cor:gaussian-smoothed-exact-sign-gen} replaces this deterministic tube assumption by a finite-horizon randomized bound.

Algorithm~\ref{alg:fedmuon} summarizes the finite-horizon FedMuon procedure analyzed here.

\begin{algorithm}[t]
\caption{FedMuon with partial participation and momentum aggregation}
\label{alg:fedmuon}
\begin{algorithmic}[1]
\State \textbf{Input:} initial model \(\bX^0\), initial momentum \(\bM^0\), rounds \(R\), local steps \(E\), participating clients \(K\), stepsize \(\eta\), memory coefficient \(\beta\), orthogonalization rule \(\Orth\), possibly as a call-dependent family \(\{\Orth_\tau\}\).
\For{\(r=0,\ldots,R-1\)}
\State Server samples \(\mathcal C_r\subset[N]\) with \(|\mathcal C_r|=K\) uniformly without replacement.
\For{each client \(i\in\mathcal C_r\) in parallel}
\State Initialize \(\bX_i^{r,0}\leftarrow\bX^r\) and \(\bM_i^{r,0}\leftarrow\bM^r\).
\For{\(k=0,\ldots,E-1\)}
\State Form a full-batch empirical gradient or mini-batch gradient \(\bG_i^{r,k}\).
\State Update \(\bM_i^{r,k+1}\leftarrow \beta\bM_i^{r,k}+(1-\beta)\bG_i^{r,k}\).
\State Set \(\bO_i^{r,k+1}\leftarrow\Orth(\bM_i^{r,k+1})\).
\State Update \(\bX_i^{r,k+1}\leftarrow\bX_i^{r,k}-\eta\bO_i^{r,k+1}\).
\EndFor
\EndFor
\State Server averages \(\bX^{r+1}\leftarrow K^{-1}\sum_{i\in\mathcal C_r}\bX_i^{r,E}\) and \(\bM^{r+1}\leftarrow K^{-1}\sum_{i\in\mathcal C_r}\bM_i^{r,E}\).
\EndFor
\State \Return \(\bX_{\rm out}\leftarrow\bX^R\).
\end{algorithmic}
\end{algorithm}

At communication round \(r\in\{0,\ldots,R-1\}\), the server has a global model \(\bX^r\) and global momentum \(\bM^r\). It samples a subset \(\mathcal C_r\subset[N]\) of size \(K\) uniformly without replacement. Each selected client initializes \(\bX_i^{r,0}=\bX^r\) and \(\bM_i^{r,0}=\bM^r\).
For local step \(k=0,\ldots,E-1\), client \(i\in\mathcal C_r\) forms a full-batch empirical gradient or a mini-batch gradient \(\bG_i^{r,k}\) and performs
\begin{align}
\bM_i^{r,k+1}
&=
\beta \bM_i^{r,k}+(1-\beta)\bG_i^{r,k},
\label{eq:momentum-update}\\
\bO_i^{r,k+1}
&=
\Orth(\bM_i^{r,k+1}),
\label{eq:orth-update}\\
\bX_i^{r,k+1}
&=
\bX_i^{r,k}-\eta \bO_i^{r,k+1}.
\label{eq:model-update}
\end{align}
The server aggregates
\begin{align}
\bX^{r+1}
&=
\frac1K\sum_{i\in\mathcal C_r}\bX_i^{r,E},
\label{eq:server-x}\\
\bM^{r+1}
&=
\frac1K\sum_{i\in\mathcal C_r}\bM_i^{r,E}.
\label{eq:server-m}
\end{align}
The output is the last global model \(\bX_{\rm out}:=\bX^R\).
When the orthogonalization rule is call-dependent, \(\Orth\) in \eqref{eq:orth-update} denotes the call-\((r,k,i)\) map.

\begin{remark}[Relation to FedMuon variants]
The update above follows the momentum-aggregation form of Liu et al.~\cite{liu2025fedmuon}, with matrix orthogonalization and a different momentum-coefficient convention. Zhang and Gao~\cite{zhang2025provable} use the convention in which the gradient has weight \(\beta\). Here, \(\beta\) is the memory coefficient and \(1-\beta\) is the new-gradient weight. Translating between the two conventions is immediate.
\end{remark}

\begin{remark}[Local--global alignment]
\label{rem:alignment}
For the alignment step of Liu et al.~\cite{liu2025fedmuon},
\[
\bX_i^{r,k+1}
=
\bX_i^{r,k}
-\eta\{(1-\alpha)\Orth(\bM_i^{r,k+1})+\alpha\boldsymbol{\Delta}_G^r\},
\]
an analogous extension starts from a separate recursion for the alignment direction \(\boldsymbol{\Delta}_G^r\). If that recursion supplies a Lipschitz constant \(L_\Delta\) on the coupled neighboring trajectory tube, the direction map can be bounded with \((1-\alpha)L_{\Orth}+\alpha L_\Delta\). The main theorem focuses on the momentum-aggregation form \eqref{eq:momentum-update}--\eqref{eq:server-m}.
\end{remark}

\section{Theory}
\label{sec:main-theorem}

\subsection{Assumptions}
\label{sec:assumptions}

The main theorem uses the assumptions explicitly cited in its statement. Assumption~\ref{ass:grad} instantiates the replacement-sensitivity constants in full-batch and mini-batch specializations.

\begin{assumption}
\label{ass:data}
The sample sets \(S_i\) are mutually independent, and samples inside each \(S_i\) are independent draws from \(\cD_i\).
\end{assumption}

\begin{assumption}
\label{ass:randomness}
All data-independent algorithmic randomness is independent of the training sample. This includes \((\bX^0,\bM^0)\), the client-participation schedule, mini-batch index randomness, Gaussian smoothing matrices used before matrix sign, and any other algorithmic randomness. Neighboring runs use the same realization of this randomness.

For a schedule-conditional statement, the participation schedule may be arbitrary after conditioning on it. After conditioning on any selected data-independent randomness, the remaining randomness is still independent of the training sample. When all data-independent randomness is conditioned on, the algorithm is deterministic as a function of the training sample.

For the deterministic random-participation bound using \(\Psi_{R,E,K,N}\), the sets \(\mathcal C_0,\ldots,\mathcal C_{R-1}\) are independent uniformly sampled subsets of \([N]\) of size \(K\). The constants used to define the amplification weights \(a_s\) are fixed independently of this random schedule.
\end{assumption}

\begin{assumption}
\label{ass:loss}
There exist constants \(B_\ell,L_\ell>0\) such that, for every \(z\) and all \(\bX,\bY\in\R^{d_1\times d_2}\),
\[
0\le \ell(\bX,z)\le B_\ell,
\qquad
|\ell(\bX,z)-\ell(\bY,z)|
\le
L_\ell\norm{\bX-\bY}_F .
\]
\end{assumption}

\begin{assumption}
\label{ass:grad}
There exist constants \(L_g,G_{\max}>0\) such that, for every \(z\) and all \(\bX,\bY\),
\[
\norm{\nabla\ell(\bX,z)-\nabla\ell(\bY,z)}_F
\le
L_g\norm{\bX-\bY}_F,
\qquad
\norm{\nabla\ell(\bX,z)}_F\le G_{\max}.
\]
\end{assumption}

\begin{assumption}
\label{ass:sampling}
After conditioning on all data-independent algorithmic randomness relevant to the gradient oracle, there are deterministic constants \(L_g\ge0\) and \(\Delta_c\ge0\), \(c\in[N]\). These constants satisfy
\[
\norm{\bG_{u,S}(\bX)-\bG_{u,S'}(\bY)}_F
\le
L_g\norm{\bX-\bY}_F+\Delta_c\,\mathbf 1\{u=c\}
\]
uniformly over all neighboring datasets \(S,S'\) differing only in one sample of client \(c\), all oracle calls, all clients \(u\in[N]\), and all \(\bX,\bY\).
The constants \(L_g\) and \(\Delta_c\) may depend on the conditioned data-independent randomness, but not on the training samples or on the particular neighboring pair.
\end{assumption}

\begin{remark}
For full-batch empirical gradients, Assumption~\ref{ass:grad} gives \(\Delta_c=2G_{\max}/n_c\). For a mini-batch of size \(b\), the uniform pathwise bound is \(\Delta_c=2G_{\max}/b\). The sharper \(2G_{\max}/n_c\) scale holds only after averaging over the mini-batch index randomness, unless one conditions on realized exposure counts as in Appendix~\ref{app:minibatch}.
\end{remark}

\begin{condition}
\label{ass:polar}
After conditioning on the data-independent algorithmic randomness used to define the orthogonalization rule, each orthogonalization call \(\tau\) uses a deterministic map \(\Orth_\tau\) on \(\R^{d_1\times d_2}\). When the map is call-independent, we write \(\Orth\).

There exist sets \(\mathcal T_\tau\subseteq\R^{d_1\times d_2}\) and a common constant \(L_{\Orth}<\infty\), fixed after this conditioning and before the training sample is drawn, such that the following hold. For every training sample \(S\), every one-sample replacement \(S'\), and every coupled pair of runs on \(S,S'\), all momentum matrices at call \(\tau\) belong to \(\mathcal T_\tau\). Moreover,
\[
\norm{\Orth_\tau(\bM)-\Orth_\tau(\bM')}_F
\le
L_{\Orth}\norm{\bM-\bM'}_F
\qquad
\text{for all }\bM,\bM'\in\mathcal T_\tau
\]
and for every call \(\tau\).
For globally Lipschitz call-independent maps one may take \(\mathcal T_\tau=\R^{d_1\times d_2}\) for all \(\tau\). For deterministic exact matrix sign, the sets \(\mathcal T_\tau\) must form a joint full-rank spectral-gap tube for all coupled neighboring trajectories.
\end{condition}

\begin{remark}
The proof compares \(S\) with every one-sample replacement \(S'\), so Condition~\ref{ass:polar} requires a Lipschitz tube containing both trajectories, unless one uses a globally Lipschitz orthogonalization map. Appendix~\ref{app:polar} gives a global Lipschitz bound for smoothed polar maps, a finite-step normalized Newton--Schulz bound with an explicit recurrence for the Lipschitz constant, a joint-gap condition for deterministic exact sign, and a finite-horizon smoothed-analysis bound for Gaussian-smoothed exact sign. For the smoothed exact-sign variant, fresh independent Gaussian perturbations produce a high-probability no-clipping event, on which the exact smoothed-sign trajectory coincides with a globally Lipschitz clipped-sign trajectory.
\end{remark}

\begin{remark}
The Frobenius-norm bound for \(\msign(\bM)\) is useful in optimization and drift-control arguments. The stability proof also uses sensitivity of directions. The one-sided risk-gap bound depends on the Lipschitz constant \(L_{\Orth}\). Deterministic exact sign requires a spectral-gap condition, while the Gaussian-smoothed exact-sign variant uses the randomized bound in Corollary~\ref{cor:gaussian-smoothed-exact-sign-gen}.
\end{remark}

\subsection{Finite-Horizon Stability Bound}
\label{sec:finite-horizon-stability}

Throughout this section assume \(R,E\ge1\), \(1\le K\le N\), \(\eta\ge0\), and \(\beta\in[0,1]\). These restrictions make \(\mathbf A\) and \(\mathbf b\) entrywise nonnegative in the stability recursion. If \(\beta\notin[0,1]\), the recursion must be rewritten with \(|\beta|\) and \(|1-\beta|\) in place of \(\beta\) and \(1-\beta\).

The proof parallels the MiMuon stability analysis, which studies generalization of single-machine Muon and a mixed Muon--momentum-SGD update~\cite{huang2026mimuon}. The stability principle is similar, but the controlled perturbation is different. Here, the perturbation comes from replacing one sample inside one client, and its effect passes through local steps, server averaging, partial participation, and heterogeneous client weights. These effects are represented by the two-state recursion below and by the realized exposure factor \(Y_i(\mathcal C)\), rather than by a single iteration count and a single training-sample size.

Define the nonnegative \(2\times2\) matrix and vector
\begin{align}
\mathbf A
&:=
\begin{pmatrix}
1+\eta L_{\Orth}(1-\beta)L_g & \eta L_{\Orth}\beta\\
(1-\beta)L_g & \beta
\end{pmatrix},
\label{eq:A-matrix}\\
\mathbf b
&:=
\begin{pmatrix}
\eta L_{\Orth}(1-\beta)\\
1-\beta
\end{pmatrix}.
\label{eq:b-vector}
\end{align}
Let \(\mathbf e_1=(1,0)^\top\), and for \(s=0,\ldots,R-1\) define
\begin{equation}
a_s
:=
\mathbf e_1^\top
\mathbf A^{E(R-1-s)}
\left(\sum_{k=0}^{E-1}\mathbf A^k\right)
\mathbf b.
\label{eq:as-def}
\end{equation}
These constants measure how a one-round perturbation in a selected client propagates to the final model \(\bX^R\); larger \(a_s\) corresponds to larger influence from round \(s\).

For a fixed participation schedule \(\mathcal C=(\mathcal C_0,\ldots,\mathcal C_{R-1})\), define the realized participation amplification for client \(i\) by
\begin{equation}
Y_i(\mathcal C)
:=
\frac1K
\sum_{s=0}^{R-1}
a_s\,\mathbf 1\{i\in\mathcal C_s\}.
\label{eq:Yi-def}
\end{equation}

For \(\delta\in(0,1)\) and \(\varepsilon=(\varepsilon_1,\ldots,\varepsilon_N)\in[0,B_\ell]^N\), define \(n_{\rm eff}^{-1}:=\sum_{i=1}^N p_i^2/n_i\).
Let \(C_{\rm hp}>0\) be a universal numerical constant. The weighted sharp-stability branch is
\begin{equation}
\mathcal H_\delta(\varepsilon)
:=
2\sum_{i=1}^Np_i\varepsilon_i
+
C_{\rm hp}
\left[
\left(
\frac{\sum_{i=1}^N n_i\varepsilon_i^2}{n_{\rm eff}}
\right)^{1/2}
\log(2n)\log\frac{2}{\delta}
+
B_\ell
\sqrt{\frac{\log(2/\delta)}{n_{\rm eff}}}
\right].
\label{eq:Hdelta}
\end{equation}
The weighted bounded-differences fallback is
\begin{equation}
\mathcal G_\delta(\varepsilon)
:=
\sum_{i=1}^Np_i\varepsilon_i
+
\sqrt{
\frac{\log(1/\delta)}{2}
\sum_{i=1}^N n_i
\left[
\left(2-\frac{p_i}{n_i}\right)\varepsilon_i
+
B_\ell\frac{p_i}{n_i}
\right]^2
}.
\label{eq:Gdelta}
\end{equation}
Finally set
\begin{equation}
\mathcal B_\delta(\varepsilon)
:=
\min\left\{
\mathcal H_{\delta/2}(\varepsilon),
\mathcal G_{\delta/2}(\varepsilon)
\right\}.
\label{eq:Bdelta}
\end{equation}
Since
\[
\left(\frac{\sum_{i=1}^N n_i\varepsilon_i^2}{n_{\rm eff}}\right)^{1/2}
\le
\left(\max_i\varepsilon_i\right)\sqrt{\frac{n}{n_{\rm eff}}},
\]
\(\mathcal H_\delta\) also implies the corresponding coarser max-stability branch.

\begin{theorem}
\label{thm:hp-gen}
Assume Assumptions~\ref{ass:data}, \ref{ass:randomness}, \ref{ass:loss}, and \ref{ass:sampling}, together with Condition~\ref{ass:polar}.

Fix a participation schedule \(\mathcal C=(\mathcal C_0,\ldots,\mathcal C_{R-1})\) with \(|\mathcal C_s|=K\).
Condition on all remaining data-independent algorithmic randomness needed for the pathwise statement, including random initialization, mini-batch index randomness, and any randomized implementation details that determine the call-dependent maps \(\Orth_\tau\). Assume that Assumption~\ref{ass:sampling} and Condition~\ref{ass:polar} hold pathwise along all coupled neighboring runs under this conditioning.

Then the constants \(\Delta_i\), \(L_{\Orth}\), and the derived quantities \(\mathbf A\), \(\mathbf b\), \(a_s\), and \(Y_i(\mathcal C)\) are deterministic and do not depend on the training sample. Let \(\bX_{\rm out}=\bX^R\), and set
\begin{equation}
\bar\varepsilon_i(\mathcal C)
:=
\min\{B_\ell,\,L_\ell\Delta_iY_i(\mathcal C)\}.
\label{eq:eps-schedule}
\end{equation}
Then, for every \(\delta\in(0,1)\),
\begin{equation}
\Pr_S\!\left\{
F(\bX_{\rm out})-\wh F_S(\bX_{\rm out})
\le
\mathcal B_\delta\bigl(
\bar\varepsilon_1(\mathcal C),\ldots,\bar\varepsilon_N(\mathcal C)
\bigr)
\right\}
\ge
1-\delta .
\label{eq:main-bound}
\end{equation}
\end{theorem}

\begin{remark}
Theorem~\ref{thm:hp-gen} controls the population-minus-empirical gap \(F(\bX_{\rm out})-\wh F_S(\bX_{\rm out})\). An absolute-deviation bound for \(|F(\bX_{\rm out})-\wh F_S(\bX_{\rm out})|\) would require a corresponding lower-tail argument.
\end{remark}

\begin{remark}
If the participation schedule and other algorithmic randomness are random but independent of the training sample, the theorem holds conditionally on every realization for which the pathwise assumptions hold. If these assumptions hold on a data-independent event of probability at least \(1-\tau\), then the same bound holds jointly over the training sample and algorithmic randomness with probability at least \(1-\tau-\delta\), with the right-hand side evaluated using the realized constants. A deterministic random-participation bound such as Corollary~\ref{cor:hp-gen-random} requires deterministic upper bounds on the amplification weights that are fixed before the random participation schedule is drawn.
\end{remark}

The orthogonalization condition is comparable to the separation requirement in earlier Muon stability work. MiMuon makes the sensitivity of exact Muon explicit through a singular-value separation parameter and then reduces the use of orthogonalization by mixing Muon with momentum SGD~\cite{huang2026mimuon}. Here we keep the FedMuon update fixed and state the generic theorem conditionally on a Lipschitz tube for the orthogonalization map. Smoothed polar maps, normalized finite-step Newton--Schulz maps, and Gaussian-smoothed exact sign are three ways to verify this condition; their formal statements are collected in Appendix~\ref{app:polar}.

\begin{corollary}
\label{cor:hp-gen-fullbatch}
Assume Assumptions~\ref{ass:data}, \ref{ass:randomness}, \ref{ass:loss}, and \ref{ass:grad}, together with Condition~\ref{ass:polar}. If each local gradient is the full-batch empirical gradient, then Assumption~\ref{ass:sampling} holds with \(\Delta_i=2G_{\max}/n_i\). For any fixed participation schedule \(\mathcal C\), set
\[
\bar\varepsilon_i^{\rm fb}(\mathcal C)
:=
\min\{B_\ell,(2G_{\max} L_\ell/n_i)Y_i(\mathcal C)\}.
\]
Then, for every \(\delta\in(0,1)\),
\[
\Pr_S\!\left\{
F(\bX_{\rm out})-\wh F_S(\bX_{\rm out})
\le
\mathcal B_\delta\bigl(
\bar\varepsilon^{\rm fb}_1(\mathcal C),\ldots,
\bar\varepsilon^{\rm fb}_N(\mathcal C)
\bigr)
\right\}
\ge
1-\delta .
\]
\end{corollary}

For random uniform participation, the following factor is a simultaneous Bernstein upper bound on \(Y_i(\mathcal C)\) over all clients. For \(\rho\in(0,1)\), define
\begin{equation}
\Psi_{R,E,K,N}(\rho)
:=
\begin{cases}
\displaystyle
\frac1N\sum_{s=0}^{R-1}a_s,
& K=N,\\[2.2ex]
\displaystyle
\frac1K
\left[
\frac KN\sum_{s=0}^{R-1}a_s
+
\sqrt{
2\frac KN\left(1-\frac KN\right)
\left(\sum_{s=0}^{R-1}a_s^2\right)
\log\frac{2N}{\rho}
}
+
\frac{a_{\max}}{3}\log\frac{2N}{\rho}
\right],
& K<N,
\end{cases}
\label{eq:Psi-def}
\end{equation}
where \(a_{\max}:=\max_{0\le s\le R-1}a_s\).

\begin{corollary}
\label{cor:hp-gen-random}
Assume the conditions of Theorem~\ref{thm:hp-gen}. Suppose \(\mathcal C_0,\ldots,\mathcal C_{R-1}\) are independent uniformly sampled subsets of \([N]\) of size \(K\). Assume, in addition, that \(L_{\Orth}\), which determines \(\mathbf A\), \(\mathbf b\), and \(a_s\), is fixed independently of the random participation schedule. Also assume that the replacement-sensitivity constants \(\Delta_i\) are fixed independently of this schedule. Equivalently, Assumption~\ref{ass:sampling} and Condition~\ref{ass:polar} hold with common constants for all schedules under consideration. Let
\begin{equation}
\varepsilon_i^{\rm det}(\rho)
:=
\min\bigl\{
B_\ell,\,
L_\ell\Delta_i\Psi_{R,E,K,N}(\rho)
\bigr\}.
\label{eq:eps-det}
\end{equation}
Then, for any \(\rho,\delta\in(0,1)\), with probability at least \(1-\rho-\delta\) over the client-participation randomness and the training sample,
\begin{equation}
F(\bX_{\rm out})-\wh F_S(\bX_{\rm out})
\le
\mathcal B_\delta\bigl(
\varepsilon_1^{\rm det}(\rho),\ldots,\varepsilon_N^{\rm det}(\rho)
\bigr).
\label{eq:main-bound-random}
\end{equation}
If, in addition, Assumption~\ref{ass:grad} holds and local gradients are full-batch empirical gradients, one may take \(\Delta_i=2G_{\max}/n_i\). If the required constants hold only on a data-independent event of probability at least \(1-\tau\), the success probability is at least \(1-\tau-\rho-\delta\).
\end{corollary}

Corollary~\ref{cor:hp-gen-random} separates two effects that are absent from the single-machine MiMuon setting~\cite{huang2026mimuon}. The replacement sensitivity \(\Delta_i\) depends on the local sample size, while \(\Psi_{R,E,K,N}\) measures how often the affected client is selected and how strongly an early perturbation is amplified by later rounds. The resulting bound therefore depends on the participation schedule and the weights \(p_i\), even in the full-batch case.

The Lipschitz condition in Theorem~\ref{thm:hp-gen} is separate from the FedMuon stability recursion. Smoothed polar maps give a global \(1/\sqrt{\lambda}\)-Lipschitz example, normalized finite-step Newton--Schulz maps give an implementable finite-polynomial example with an explicit coefficient-dependent Lipschitz constant, and Gaussian-smoothed exact sign replaces a deterministic neighboring-tube gap by a finite-horizon no-clipping event. The detailed statements are deferred to Appendix~\ref{app:polar}, specifically Corollaries~\ref{cor:fedmuon-smoothed-polar-hp}, \ref{cor:gaussian-smoothed-exact-sign-gen}, and \ref{cor:fedmuon-ns-hp}.

The exact-sign specialization is a randomized finite-horizon statement for a perturbed matrix-sign update. It adds fresh Gaussian perturbations and pays an additional failure probability for a rare near-rank-deficient event. The scalar obstruction in Proposition~\ref{prop:exact-sign-obstruction} shows the role of this randomization or an alternative deterministic neighboring-tube spectral gap.

The proof of Theorem~\ref{thm:hp-gen} is in Appendix~\ref{app:proof-main}. Appendix~\ref{app:proof-fullbatch-no-sampling} proves Corollary~\ref{cor:hp-gen-fullbatch}, Appendix~\ref{app:proof-random-corollary} proves Corollary~\ref{cor:hp-gen-random}, and Appendix~\ref{app:polar} gives the orthogonalization specializations.

Theorem~\ref{thm:hp-gen} depends on \(N,K,R,E\) through the realized quantities \(Y_i(\mathcal C)\); Corollary~\ref{cor:hp-gen-random} replaces them by \(\Psi_{R,E,K,N}\) under deterministic amplification constants. The sample sizes and client weights enter through \(\Delta_i\), \(n_{\rm eff}\), and the weighted bounded-differences fallback. The hyperparameters \(\eta,\beta,L_g,L_{\Orth}\) enter through \(\mathbf A,\mathbf b\), while confidence parameters enter logarithmically. Dimension enters through orthogonalization constants such as \(L_{\Orth}\), and the risk-generalization proof proceeds through sample-replacement sensitivity rather than a bounded client-drift parameter.

For uniform clients and full participation, Appendix~\ref{app:corollary-proofs} derives the simplified order
\[
\widetilde{\mathcal O}\!\left(
\frac{G_{\max}L_\ell C_{\rm amp}\eta L_{\Orth}ER}{n}
+
\frac{B_\ell}{\sqrt n}
\right)
\]
under an explicit finite-horizon amplification condition. The appendix also records the corresponding partial-participation order and the spectral-radius calculation for the two-state amplification matrix. These rates serve as interpretations of Theorem~\ref{thm:hp-gen}.

\section{Experiments}
\label{sec:experiments}

The experiments combine controlled diagnostics aligned with the stability theorem and complementary CNN optimizer diagnostics. The controlled matrix-classifier experiments use bounded losses, full-batch local gradients, partial participation, and smoothed polar orthogonalization, matching the assumptions used in the displayed stability bound. The CNN experiments evaluate the same orthogonalized update in a standard nonlinear training setting. The reported runs were conducted on a single NVIDIA RTX PRO 6000 Blackwell GPU with 96 GB of memory.

The controlled diagnostics evaluate the quantities appearing in the theorem. Suite A1 varies the smooth-polar regularization and stepsize, thereby changing both the effective Lipschitz constant of the orthogonalization map and the finite-horizon amplification. Suite A2 varies the number of training samples and assesses whether the evaluated bound and neighboring-model perturbation decrease with sample size. Suite A3 varies participation and label-skew settings to examine how realized exposure \(Y_i(\mathcal C)\) is reflected in the measured perturbation. Appendix~\ref{app:experiment-details} gives the complete grids, seeds, and aggregate values.

Figure~\ref{fig:cnn_loss_accuracy_curves} reports training loss, held-out loss, and test accuracy trajectories in the CNN optimizer diagnostics. Figure~\ref{fig:fedmuon_phase_diagnostics} reports the phase behavior over smooth-polar regularization and stepsize, and Figure~\ref{fig:fedmuon_scaling_exposure} isolates the sample-size and participation-exposure effects predicted by the stability recursion.

\begin{figure}[!htbp]
\centering
\includegraphics[width=0.90\linewidth]{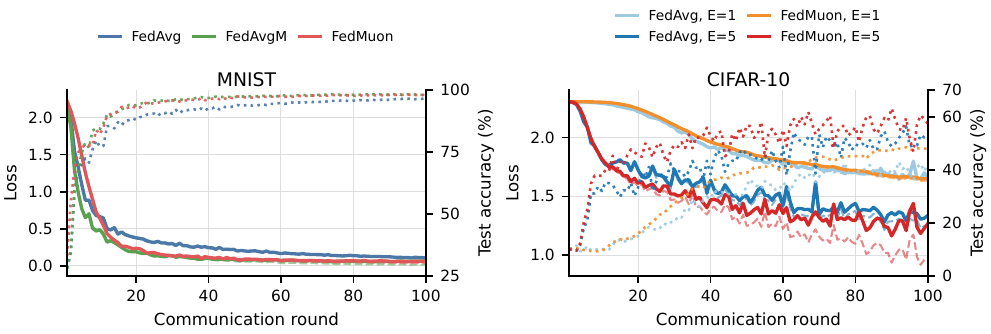}
\caption{CNN loss and accuracy trajectories on MNIST and CIFAR-10. Solid curves denote held-out loss, dashed curves denote training loss, and dotted curves denote test accuracy on the right axis. MNIST averages seeds \(0,1,2\); CIFAR-10 reports seed \(0\).}
\label{fig:cnn_loss_accuracy_curves}
\end{figure}

The CNN curves place the orthogonalized update in a familiar nonlinear training regime. On MNIST, FedMuon reaches a terminal accuracy comparable to FedAvgM, with a held-out loss close to the FedAvgM trajectory. On CIFAR-10, the seed-\(0\) trajectory shows that increasing the number of local steps improves both FedAvg and FedMuon, and FedMuon attains lower terminal training and test losses in the displayed configurations. These diagnostics complement the controlled matrix-classifier experiments by assessing optimizer behavior in CNN training.

\FloatBarrier

\begin{figure}[H]
\centering
\subfloat[MNIST bound.]{%
  \includegraphics[width=0.242\linewidth]{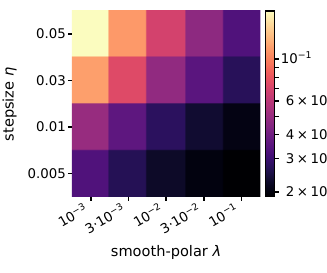}
}%
\hfill
\subfloat[MNIST stability.]{%
  \includegraphics[width=0.242\linewidth]{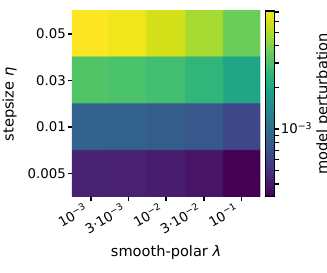}
}%
\hfill
\subfloat[CIFAR-10 bound.]{%
  \includegraphics[width=0.242\linewidth]{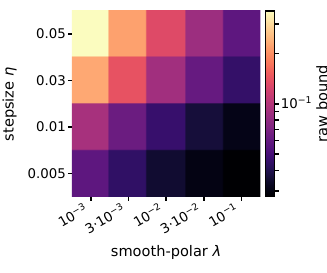}
}%
\hfill
\subfloat[CIFAR-10 stability.]{%
  \includegraphics[width=0.242\linewidth]{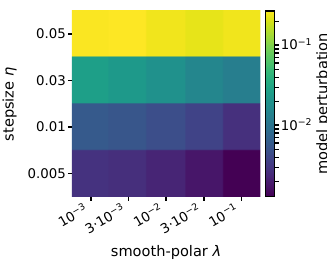}
}%
\caption{Smooth-polar FedMuon phase diagnostics. Panels (a) and (c) report the bound from Theorem~\ref{thm:hp-gen}; panels (b) and (d) report neighboring-dataset model perturbation. Larger stepsizes and smaller smooth-polar regularization increase both quantities.}
\label{fig:fedmuon_phase_diagnostics}
\end{figure}

The phase diagnostics show the same qualitative dependence in the evaluated bound and in direct neighboring-dataset perturbations. Larger stepsizes increase the amplification matrix in \eqref{eq:A-matrix}, while smaller smooth-polar regularization increases the Lipschitz constant \(L_{\Orth}=1/\sqrt{\lambda}\). The resulting increase in both quantities is consistent with the stability recursion and shows that orthogonalization regularity has a visible finite-horizon effect in these diagnostics.

\begin{figure}[H]
\centering
\subfloat[MNIST sample scaling.]{%
  \includegraphics[width=0.325\linewidth]{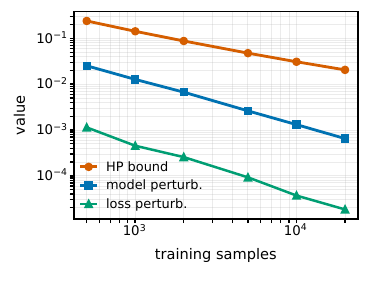}
}%
\hfill
\subfloat[CIFAR-10 sample scaling.]{%
  \includegraphics[width=0.325\linewidth]{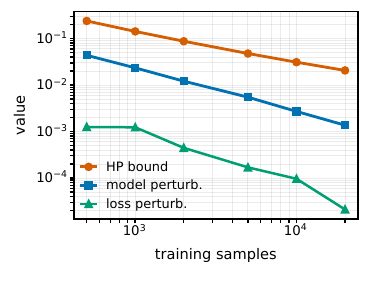}
}%
\hfill
\subfloat[Exposure vs. model perturb.]{%
  \includegraphics[width=0.325\linewidth]{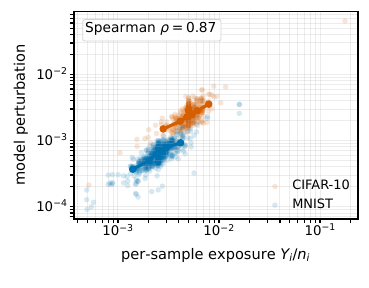}
}%
\caption{Finite-sample and participation diagnostics. Panels (a) and (b) show decreasing bounds and perturbations as the training subset grows. Panel (c) relates per-sample exposure \(Y_i(\mathcal C)/n_i\) to neighboring-model perturbation, with Spearman correlation \(\rho=0.87\).}
\label{fig:fedmuon_scaling_exposure}
\end{figure}

The scaling plots summarize the finite-sample behavior in the controlled setting. As the training subset grows, both datasets exhibit smaller evaluated bounds and smaller neighboring-model perturbations, matching the role of \(\Delta_i=2G_{\max}/n_i\) in the full-batch specialization. The exposure plot separates this sample-size effect from schedule effects: samples belonging to clients with larger realized exposure \(Y_i(\mathcal C)/n_i\) tend to produce larger neighboring perturbations. This is the client-level mechanism retained by Theorem~\ref{thm:hp-gen} and collapsed into \(\Psi_{R,E,K,N}\) only in the random-participation corollary.

\section{Conclusion}
\label{sec:conclusion}

This paper gives a finite-horizon, one-sided high-probability risk-gap bound for FedMuon under partial client participation. The result separates the stability recursion from the orthogonalization analysis: smoothed polar maps and normalized finite-step Newton--Schulz maps provide global Lipschitz bounds, deterministic exact matrix sign is covered by a joint neighboring-trajectory spectral gap, and Gaussian-smoothed exact sign admits a finite-horizon randomized bound. The bound keeps the realized participation amplification explicit and derives simplified rates under verifiable finite-horizon amplification control. The experiments report the quantities appearing in the bound.

\bibliographystyle{unsrt}
\bibliography{reference}

\clearpage
\def\MAINFILE{1}
\ifdefined\MAINFILE
\else
\documentclass[11pt]{article}
\usepackage[margin=1in]{geometry}
\usepackage{amsmath,amssymb,amsfonts,amsthm,mathtools}
\usepackage{enumitem}
\usepackage{array}
\usepackage{graphicx}
\usepackage{xcolor}
\usepackage{hyperref}
\hypersetup{colorlinks=true,linkcolor=blue,urlcolor=blue,citecolor=blue}

\newtheorem{theorem}{Theorem}[section]
\newtheorem{corollary}[theorem]{Corollary}
\newtheorem{lemma}[theorem]{Lemma}
\newtheorem{proposition}[theorem]{Proposition}
\newtheorem{assumption}{Assumption}[section]
\newtheorem{condition}{Condition}[section]
\newtheorem{remark}{Remark}[section]
\newtheorem{definition}{Definition}[section]

\DeclareMathOperator{\msign}{msign}
\DeclareMathOperator{\rank}{rank}
\DeclareMathOperator{\Tr}{Tr}

\newcommand{\R}{\mathbb{R}}
\newcommand{\E}{\mathbb{E}}
\newcommand{\Prb}{\mathbb{P}}
\newcommand{\cD}{\mathcal{D}}
\newcommand{\cI}{\mathcal{I}}
\newcommand{\Orth}{\mathrm{Orth}}
\newcommand{\wh}{\widehat}
\newcommand{\bX}{\mathbf X}
\newcommand{\bY}{\mathbf Y}
\newcommand{\bM}{\mathbf M}
\newcommand{\bO}{\mathbf O}
\newcommand{\bG}{\mathbf G}
\newcommand{\bU}{\mathbf U}
\newcommand{\bV}{\mathbf V}
\newcommand{\bSigma}{\boldsymbol{\Sigma}}
\newcommand{\ip}[2]{\langle #1,#2\rangle}
\newcommand{\norm}[1]{\left\lVert #1\right\rVert}
\newcolumntype{P}[1]{>{\raggedright\arraybackslash}p{#1}}

\begin{document}
\title{Appendix: Proofs for the FedMuon One-Sided High-Probability Risk-Gap Bound}
\author{}
\date{}
\maketitle
\fi

\appendix

\section{Additional Related Work}
\label{sec:related-work}

FedAvg established periodic server averaging of local stochastic updates as a central algorithmic template for communication-efficient federated learning~\cite{mcmahan2017communication}. Local SGD theory formalizes how multiple local steps can preserve convergence while reducing communication~\cite{stich2018local}, and related analyses explain model averaging and linear speedup phenomena in parallel or communication-efficient stochastic optimization~\cite{yu2019parallel,yu2019linear}. Variance-reduction and control-variate methods such as SCAFFOLD address client drift under heterogeneous data~\cite{karimireddy2020scaffold}, while adaptive federated optimizers extend server-side optimization beyond plain averaging~\cite{reddi2020adaptive}. Momentum methods have also been analyzed for non-IID federated learning under appropriate assumptions~\cite{cheng2024momentum}. Recent work on stochastic weight averaging in heterogeneous FL explicitly targets generalization improvements rather than only communication or stationarity~\cite{liu2025fedswa}. These works provide the federated-optimization context for the finite-sample FedMuon setting considered here.

Muon updates replace an unconstrained gradient or momentum direction by an orthogonalized matrix direction~\cite{jordan2024muon}. The method is closely connected to norm-constrained linear minimization oracles~\cite{pethick2025training} and to the view that changing the update norm changes optimization geometry~\cite{bernstein2024old}. Practical implementations often approximate the matrix sign or polar factor by Newton--Schulz iterations; modern Muon-specific treatments study optimal matrix-sign methods, Newton--Schulz Muon, and inexact Muon updates~\cite{schulz1933iterative,amsel2025polar,kim2026convergence,shulgin2025beyond}. The numerical linear algebra literature shows that polar decomposition can be sensitive near rank-deficient matrices~\cite{higham1986computing}. Recent optimization analyses also connect Muon to spectral flattening and spectral preconditioning effects~\cite{ma2026preconditioning,nguyen2026spectral}. These results motivate the trajectory-wise Lipschitz condition imposed on the orthogonalization map.

Recent FedMuon analyses establish optimization guarantees for Muon-style methods in federated settings. Zhang and Gao analyze Muon in federated optimization under standard and heavy-tailed noise assumptions~\cite{zhang2025provable}. Takezawa et al. identify an LMO-bias failure mode of direct LocalMuon and propose a bias-corrected FedMuon variant~\cite{takezawa2025fedmuon}. Liu et al. analyze a matrix-orthogonalized FedMuon method with partial participation, local--global alignment, and momentum aggregation~\cite{liu2025fedmuon}. In the single-machine setting, Muon theory spans the original optimizer formulation, LMO and update-norm interpretations, matrix-sign approximation, and convergence analyses for Muon variants~\cite{jordan2024muon,pethick2025training,bernstein2024old,amsel2025polar,chang2025convergence}. MiMuon is the closest generalization work: it analyzes Muon by algorithmic stability and introduces a mixed Muon--momentum-SGD update to improve the generalization bound~\cite{huang2026mimuon}. The present paper shares the stability viewpoint but studies a federated, partial-participation, one-sided high-probability risk-gap bound under Lipschitz-controlled orthogonalization, with client weights and schedule exposure appearing explicitly. Other recent Muon-adjacent proposals, including adaptive, tensorized, block-periodic, and distributed orthonormalized variants, further illustrate the importance of the orthogonalization map but are not the main comparison target here~\cite{si2025adamuon,zhang2026teon,liu2026muon,khaled2025muonbp,ahn2025dion}.

The proof follows the algorithmic-stability route from training-set perturbations to generalization. Uniform stability was introduced as a general tool for proving learning guarantees~\cite{bousquet2002stability}, and later work used it to analyze stochastic gradient methods~\cite{hardt2016train}. Feldman and Vondrak refined this line with sharper generalization bounds for uniformly stable algorithms and high-probability variants~\cite{feldman2018generalization,feldman2019high}. The final concentration step uses a weighted bounded-differences argument in the spirit of McDiarmid's method~\cite{mcdiarmid1989method}. Compared with these general stability results, the main additional component is the neighboring-dataset recursion for momentum aggregation, partial client participation, and the non-globally-Lipschitz Muon orthogonalization map.

Heavy-tailed noise has been studied in nonconvex stochastic optimization without gradient clipping~\cite{liu2025nonconvex} and in distributed optimization settings~\cite{lee2025efficient}. The FedMuon analysis of Zhang and Gao also includes heavy-tailed noise assumptions~\cite{zhang2025provable}. Separately, Byzantine-robust distributed stochastic optimization with data heterogeneity studies adversarially corrupted workers rather than ordinary random client sampling~\cite{JMLR:v26:25-0613}. These directions are adjacent to the present problem because noise, heterogeneity, system constraints, and data-dependent updates affect convergence and evaluation arguments. The theorem treats a non-Byzantine finite-sample FL setting and uses bounded loss for the high-probability risk-gap conversion.

Our stability analysis and the use of Muon-type optimizers are directly relevant to a range of modern training settings, including scalable and equilibrated Muon training, momentum-gradient alignment, quantization-aware training, adaptive gradient reweighting, LLM fine-tuning with LoRA, agent policy optimization, preference optimization, token-level credit assignment, and segmentation~\cite{liu2025muon,chang2026muoneq,chang2026mgup,zhang2026forget,Tong_2025_ICCV,chang2026calibrating,li2026reason,li2025inspo,li2026oppo,li2025kg}. These applications motivate finite-sample generalization guarantees for structured, geometry-aware optimizers beyond the specific FedMuon setting studied here.

\section{Experimental Details}
\label{app:experiment-details}

This section records the numerical configuration values used in Section~\ref{sec:experiments}. The data are MNIST and CIFAR-10 from the official \texttt{torchvision} splits. Client partitions use Dirichlet label skew, and clients are sampled without replacement in each round. The reported runs were executed on a single NVIDIA RTX PRO 6000 Blackwell GPU with 96 GB memory.

\begin{table}[!htbp]
\centering
\scriptsize
\setlength{\tabcolsep}{2.2pt}
\resizebox{\linewidth}{!}{%
\begin{tabular}{llrrrrrrlllrr}
\hline
Suite & Data & Train & Test & \(N\) & \(K\) & \(R\) & \(E\) & \(\alpha\) & \(\eta\) & \(\lambda\) & Seeds & Rep. \\
\hline
A1 & MNIST & 20000 & 5000 & 20 & 5 & 50 & 1 & 0.5 & 0.005/0.01/0.03/0.05 & 1e-3/3e-3/1e-2/3e-2/1e-1 & 0--2 & 10 \\
A1 & CIFAR-10 & 10000 & 5000 & 20 & 5 & 50 & 1 & 0.5 & 0.005/0.01/0.03/0.05 & 1e-3/3e-3/1e-2/3e-2/1e-1 & 0--2 & 10 \\
A2 & MNIST & 0.5/1/2/5/10/20k & 5000 & 20 & 5 & 50 & 1 & 0.5 & 0.01 & 0.1 & 0--4 & 20 \\
A2 & CIFAR-10 & 0.5/1/2/5/10/20k & 5000 & 20 & 5 & 50 & 1 & 0.5 & 0.01 & 0.1 & 0--4 & 10 \\
A3 & MNIST & 20000 & 5000 & 20 & 2/5/10/20 & 50 & 1 & 0.1/0.5/10 & 0.01 & 0.1 & 0--2 & 20 \\
A3 & CIFAR-10 & 10000 & 5000 & 20 & 2/5/10/20 & 50 & 1 & 0.1/0.5/10 & 0.01 & 0.1 & 0--2 & 10 \\
\hline
\end{tabular}
}
\caption{Numerical grid for the controlled matrix-classifier diagnostics. \(N\) is the number of clients, \(K\) the number of participating clients per round, and Rep. the number of neighboring-dataset replacement trials per run.}
\label{tab:matrix-diagnostic-details}
\end{table}

\begin{table}[!htbp]
\centering
\scriptsize
\setlength{\tabcolsep}{4pt}
\begin{tabular}{llrrrrr}
\hline
Data & Suite & Runs & Gap \(\le\) bound & Mean gap & Mean bound & Mean perturb. \\
\hline
CIFAR-10 & A1 & 60 & 1.00 & 0.01421 & 0.09361 & \(5.06{\times}10^{-3}\) \\
CIFAR-10 & A2 & 30 & 1.00 & 0.02948 & 0.09508 & \(1.84{\times}10^{-4}\) \\
CIFAR-10 & A3 & 36 & 1.00 & 0.00508 & 0.03162 & \(3.73{\times}10^{-5}\) \\
MNIST & A1 & 60 & 1.00 & -0.00729 & 0.05061 & \(1.18{\times}10^{-5}\) \\
MNIST & A2 & 30 & 1.00 & 0.00457 & 0.09495 & \(9.96{\times}10^{-5}\) \\
MNIST & A3 & 36 & 1.00 & -0.00518 & 0.02063 & \(7.06{\times}10^{-6}\) \\
\hline
\end{tabular}
\caption{Aggregate values behind Figures~\ref{fig:fedmuon_phase_diagnostics} and~\ref{fig:fedmuon_scaling_exposure}. The gap column is test loss minus weighted training loss, and the perturbation column is the mean neighboring-model Frobenius perturbation.}
\label{tab:matrix-diagnostic-summary}
\end{table}

\begin{table}[!htbp]
\centering
\scriptsize
\setlength{\tabcolsep}{4pt}
\begin{tabular}{llrrrrrrr}
\hline
Data & Method & \(E\) & Runs & \(\eta\) & Batch & Train loss & Test loss & Acc. (\%) \\
\hline
MNIST & FedAvg & 5 & 3 & 0.05 & 64 & 0.1019 & 0.1065 & 96.73 \\
MNIST & FedAvgM & 5 & 3 & 0.05 & 64 & 0.0161 & 0.0635 & 98.19 \\
MNIST & FedMuon & 5 & 3 & 0.01 & 64 & 0.0390 & 0.0581 & 98.19 \\
CIFAR-10 & FedAvg & 1 & 1 & 0.03 & 128 & 1.6575 & 1.6584 & 42.26 \\
CIFAR-10 & FedAvg & 5 & 1 & 0.03 & 128 & 1.2617 & 1.3372 & 52.24 \\
CIFAR-10 & FedMuon & 1 & 1 & 0.005 & 128 & 1.6325 & 1.6499 & 47.90 \\
CIFAR-10 & FedMuon & 5 & 1 & 0.005 & 128 & 0.9807 & 1.2676 & 57.40 \\
\hline
\end{tabular}
\caption{Terminal values for the CNN optimizer-diagnostic curves in Figure~\ref{fig:cnn_loss_accuracy_curves}. MNIST entries average seeds \(0,1,2\); CIFAR-10 entries use seed \(0\). FedMuon uses \(\beta=0.9\), five Newton--Schulz steps, and \(\epsilon_{\rm ns}=10^{-7}\).}
\label{tab:cnn-diagnostic-details}
\end{table}

All tables and figures in the experimental section are computed from the recorded training trajectories under the configurations listed above. Repeated runs are averaged over the stated seeds.

\section{Auxiliary Stability and Concentration Lemmas}
\label{app:aux}

Throughout the appendix, a sample index is denoted by \(a=(i,j)\in\cI\), and its empirical-risk weight is
\[
w_a=w_{i,j}:=\frac{p_i}{n_i}.
\]
For a dataset \(S\), let \(\mathcal A(S)\) be the output of a deterministic algorithm. In the main proof, the algorithm becomes deterministic after conditioning on the client-participation schedule and the remaining data-independent algorithmic randomness.

\begin{definition}
Two datasets \(S,S'\) are \(a=(i,j)\)-neighbors if they differ only in the sample \(z_{i,j}\). An algorithm \(\mathcal A\) has weighted uniform stability coefficients \(\{\varepsilon_i\}_{i=1}^N\) if for every \(a=(i,j)\), every \(a\)-neighboring pair \(S,S'\), and every test example \(z\),
\[
|\ell(\mathcal A(S),z)-\ell(\mathcal A(S'),z)|
\le
\varepsilon_i.
\]
\end{definition}

\begin{lemma}
\label{lem:expected-stability}
If \(\mathcal A\) has stability coefficients \(\{\varepsilon_i\}_{i=1}^N\), then
\[
\E_S\!\left[
F(\mathcal A(S))-\wh F_S(\mathcal A(S))
\right]
\le
\sum_{i=1}^N p_i\varepsilon_i.
\]
\end{lemma}

\begin{proof}
Let \(\tilde z_{i,j}\sim\cD_i\) be an independent ghost sample, and let \(S^{(i,j)}\) be \(S\) with \(z_{i,j}\) replaced by \(\tilde z_{i,j}\). Then
\[
\E F(\mathcal A(S))
=
\sum_{i=1}^N\sum_{j=1}^{n_i}w_{i,j}
\E\bigl[\ell(\mathcal A(S),\tilde z_{i,j})\bigr],
\]
and
\[
\E \wh F_S(\mathcal A(S))
=
\sum_{i=1}^N\sum_{j=1}^{n_i}w_{i,j}
\E\bigl[\ell(\mathcal A(S),z_{i,j})\bigr].
\]
Because \(S^{(i,j)}\) has the same distribution as \(S\), and \(\tilde z_{i,j}\) plays the role of the replaced sample in \(S^{(i,j)}\),
\[
\E\bigl[\ell(\mathcal A(S),z_{i,j})\bigr]
=
\E\bigl[\ell(\mathcal A(S^{(i,j)}),\tilde z_{i,j})\bigr].
\]
Thus
\begin{align*}
\E\!\left[F(\mathcal A(S))-\wh F_S(\mathcal A(S))\right]
&=
\sum_{i,j}w_{i,j}
\E\!\left[
\ell(\mathcal A(S),\tilde z_{i,j})
-
\ell(\mathcal A(S^{(i,j)}),\tilde z_{i,j})
\right]\\
&\le
\sum_{i,j}w_{i,j}\varepsilon_i
=
\sum_{i=1}^N p_i\varepsilon_i.
\end{align*}
\end{proof}

\begin{lemma}
\label{lem:weighted-bkz-moment}
Let \(Z_1,\ldots,Z_m\) be independent random variables, not necessarily identically distributed. For \(a\in[m]\), let \(g_a=g_a(Z_1,\ldots,Z_m)\) be a real-valued measurable function. Let \(w_a\ge0\) and \(\alpha_a\ge0\). Assume that all conditional expectations below are well-defined and that, for every \(a\in[m]\),
\[
\E[g_a\mid Z_{-a}]=0,
\qquad
\left|\E[g_a\mid Z_a]\right|\le M
\quad\text{a.s.}
\]
Assume also that, for every \(a\ne b\), replacing only \(Z_b\) changes \(g_a\) by at most \(\alpha_b\). Then there is a universal numerical constant \(C_0>0\) such that, for every \(r\ge2\),
\[
\left\|
\sum_{a=1}^m w_ag_a
\right\|_r
\le
C_0\left[
M\sqrt r\left(\sum_{a=1}^m w_a^2\right)^{1/2}
+
r\lceil\log_2(2m)\rceil
\left(\sum_{a=1}^m w_a^2\right)^{1/2}
\left(\sum_{a=1}^m \alpha_a^2\right)^{1/2}
\right].
\]
\end{lemma}

\begin{proof}
For a sigma-field \(\mathcal F\), write
\[
\norm{X}_{r\mid\mathcal F}:=\left(\E[|X|^r\mid\mathcal F]\right)^{1/r}.
\]
We first record two inequalities used in the proof.

First, suppose \(H=H(Y_1,\ldots,Y_s)\), where \(Y_1,\ldots,Y_s\) are independent conditional on \(\mathcal F\), and replacing \(Y_j\) changes \(H\) by at most \(c_j\). Then McDiarmid's inequality applied conditionally on \(\mathcal F\) gives
\[
\Prb\left\{
|H-\E[H\mid\mathcal F]|\ge t
\,\middle|\,
\mathcal F
\right\}
\le
2\exp\left(
-\frac{2t^2}{\sum_{j=1}^s c_j^2}
\right).
\]
Integrating this tail bound yields, for all \(r\ge2\),
\begin{equation}
\norm{H-\E[H\mid\mathcal F]}_{r\mid\mathcal F}
\le
C_{\rm bd}\sqrt r
\left(\sum_{j=1}^s c_j^2\right)^{1/2}.
\label{eq:conditional-bd-lr}
\end{equation}

Second, suppose \(X_1,\ldots,X_s\) are conditionally independent and conditionally centered given \(\mathcal F\). Then, for deterministic coefficients \(\lambda_j\),
\begin{equation}
\norm{\sum_{j=1}^s\lambda_jX_j}_{r\mid\mathcal F}
\le
C_{\rm mz}\sqrt r
\left(
\sum_{j=1}^s\lambda_j^2\norm{X_j}_{r\mid\mathcal F}^2
\right)^{1/2}.
\label{eq:conditional-mz}
\end{equation}
Indeed, conditional symmetrization gives a factor \(2\) and introduces independent Rademacher signs \(\xi_j\). Conditional on \(X_1,\ldots,X_s,\mathcal F\), Khintchine's inequality gives
\[
\left(\E_\xi\left|
\sum_{j=1}^s\lambda_j\xi_jX_j
\right|^r\right)^{1/r}
\le
C\sqrt r
\left(\sum_{j=1}^s\lambda_j^2X_j^2\right)^{1/2}.
\]
Finally,
\[
\norm{
\left(\sum_j\lambda_j^2X_j^2\right)^{1/2}
}_{r\mid\mathcal F}^2
=
\norm{
\sum_j\lambda_j^2X_j^2
}_{r/2\mid\mathcal F}
\le
\sum_j\lambda_j^2\norm{X_j}_{r\mid\mathcal F}^2,
\]
which proves \eqref{eq:conditional-mz}.

Choose \(J\) such that \(2^{J-1}<m\le2^J\). Add dummy coordinates with \(w_a=\alpha_a=g_a=0\) so that the padded number of coordinates is \(2^J\). This changes neither side of the desired inequality and satisfies \(J\le\lceil\log_2(2m)\rceil\).

Let \(\mathcal P_\ell\), \(\ell=0,\ldots,J\), be the dyadic partition of the padded index set into blocks of size \(2^\ell\). For \(a\), let \(B_\ell(a)\in\mathcal P_\ell\) denote the block containing \(a\), and define
\[
g_a^\ell
:=
\E[g_a\mid Z_a,Z_{-B_\ell(a)}].
\]
Then \(g_a^0=g_a\), \(g_a^J=\E[g_a\mid Z_a]\), and
\[
g_a
=
\E[g_a\mid Z_a]
+
\sum_{\ell=0}^{J-1}
\left(g_a^\ell-g_a^{\ell+1}\right).
\]

We first bound the terminal term. Let \(h_a(Z_a):=\E[g_a\mid Z_a]\). Since
\[
\E h_a=\E g_a=\E\,\E[g_a\mid Z_{-a}]=0,
\]
the variables \(h_a(Z_a)\) are independent, centered, and bounded by \(M\). Applying \eqref{eq:conditional-mz} with the trivial sigma-field gives
\begin{equation}
\left\|
\sum_a w_a h_a(Z_a)
\right\|_r
\le
C M\sqrt r
\left(\sum_a w_a^2\right)^{1/2}.
\label{eq:terminal-term-bound}
\end{equation}

Fix a level \(\ell<J\). For \(B\in\mathcal P_\ell\), let \(B^\star\) be the dyadic sibling of \(B\), and let \(P=B\cup B^\star\in\mathcal P_{\ell+1}\). For \(a\in B\), define
\[
D_a^\ell:=g_a^\ell-g_a^{\ell+1}.
\]
Then \(D_a^\ell\) is measurable with respect to \(\sigma(Z_a,Z_{-B})\). Conditional on \(Z_{-B}\), the variables \(\{D_a^\ell:a\in B\}\) depend on disjoint coordinates \(\{Z_a:a\in B\}\), and hence are conditionally independent. They are also conditionally centered. Indeed,
\[
\E[g_a^\ell\mid Z_{-B}]
=
\E[g_a\mid Z_{-B}]
=
\E\!\left[\E[g_a\mid Z_{-a}]\mid Z_{-B}\right]
=
0,
\]
and, since \(Z_{-P}\subseteq Z_{-a}\),
\[
\E[g_a^{\ell+1}\mid Z_{-B}]
=
\E[g_a\mid Z_{-P}]
=
\E\!\left[\E[g_a\mid Z_{-a}]\mid Z_{-P}\right]
=
0.
\]
Therefore, by \eqref{eq:conditional-mz},
\begin{align}
\left\|
\sum_{a\in B}w_aD_a^\ell
\right\|_r
&\le
C\sqrt r
\left\|
\left(
\sum_{a\in B}w_a^2
\norm{D_a^\ell}_{r\mid Z_{-B}}^2
\right)^{1/2}
\right\|_r
\nonumber\\
&\le
C\sqrt r
\left(
\sum_{a\in B}w_a^2
\norm{D_a^\ell}_r^2
\right)^{1/2}.
\label{eq:block-mz-step}
\end{align}
The last inequality uses Minkowski in \(L_{r/2}\) and the identity
\[
\norm{\norm{D_a^\ell}_{r\mid Z_{-B}}}_r=\norm{D_a^\ell}_r.
\]

It remains to bound \(\norm{D_a^\ell}_r\). Conditional on \(Z_a\) and \(Z_{-P}\), the quantity \(g_a^\ell\) is a function of the sibling coordinates \(Z_{B^\star}\), and \(g_a^{\ell+1}\) is its conditional expectation over \(Z_{B^\star}\). Replacing \(Z_b\), \(b\in B^\star\), changes this conditional function by at most \(\alpha_b\), because conditional expectation preserves deterministic bounded-difference constants. Applying \eqref{eq:conditional-bd-lr} conditionally on \(\sigma(Z_a,Z_{-P})\) gives
\[
\norm{D_a^\ell}_{r\mid Z_a,Z_{-P}}
\le
C_{\rm bd}\sqrt r
\left(\sum_{b\in B^\star}\alpha_b^2\right)^{1/2}.
\]
Consequently,
\[
\norm{D_a^\ell}_r
\le
C_{\rm bd}\sqrt r
\left(\sum_{b\in B^\star}\alpha_b^2\right)^{1/2}.
\]
Combining this with \eqref{eq:block-mz-step}, we obtain
\[
\left\|
\sum_{a\in B}w_aD_a^\ell
\right\|_r
\le
C r
\left(\sum_{a\in B}w_a^2\right)^{1/2}
\left(\sum_{b\in B^\star}\alpha_b^2\right)^{1/2}.
\]

Define
\[
W_B:=\left(\sum_{a\in B}w_a^2\right)^{1/2},
\qquad
A_B:=\left(\sum_{a\in B}\alpha_a^2\right)^{1/2}.
\]
By the triangle inequality over blocks at level \(\ell\),
\[
\left\|
\sum_a w_a(g_a^\ell-g_a^{\ell+1})
\right\|_r
\le
C r\sum_{B\in\mathcal P_\ell} W_BA_{B^\star}.
\]
For each parent \(P\in\mathcal P_{\ell+1}\) with children \(B\) and \(B^\star\),
\[
W_BA_{B^\star}+W_{B^\star}A_B
\le
\left(W_B^2+W_{B^\star}^2\right)^{1/2}
\left(A_B^2+A_{B^\star}^2\right)^{1/2}
=
W_PA_P.
\]
Hence
\[
\sum_{B\in\mathcal P_\ell} W_BA_{B^\star}
\le
\sum_{P\in\mathcal P_{\ell+1}}W_PA_P
\le
\left(\sum_a w_a^2\right)^{1/2}
\left(\sum_a \alpha_a^2\right)^{1/2}.
\]
Summing over \(\ell=0,\ldots,J-1\) and combining with \eqref{eq:terminal-term-bound} proves the lemma.
\end{proof}

\begin{lemma}
\label{lem:moment-to-tail}
Let \(X\) be a real random variable. Suppose that for all \(r\ge2\),
\[
\norm{X}_r\le ar+b\sqrt r .
\]
Then there is a universal numerical constant \(C>0\) such that, for every \(\delta\in(0,1)\), with probability at least \(1-\delta\),
\[
X
\le
C\left[
a\log\frac{2}{\delta}
+
b\sqrt{\log\frac{2}{\delta}}
\right].
\]
\end{lemma}

\begin{proof}
Let \(r=\max\{2,\log(1/\delta)\}\). Markov's inequality gives
\[
\Prb\{|X|\ge e\norm{X}_r\}\le e^{-r}\le \delta.
\]
For every \(\delta\in(0,1)\),
\[
r\le 3\log\frac{2}{\delta}.
\]
The claim follows after absorbing the numerical factors \(e\), \(3\), and \(\sqrt3\) into the universal constant.
\end{proof}

\begin{theorem}
\label{thm:weighted-sharp-stability}
Assume \(0\le \ell\le B_\ell\). Let \(\mathcal A\) be a deterministic algorithm with weighted uniform stability coefficients \(\varepsilon_i\in[0,B_\ell]\), meaning that replacing one sample of client \(i\) changes every test loss by at most \(\varepsilon_i\). Let
\[
n_{\rm eff}^{-1}:=\sum_{i=1}^N\frac{p_i^2}{n_i}.
\]
Then there is a universal numerical constant \(C_{\rm hp}>0\) such that, for every \(\delta\in(0,1)\), with probability at least \(1-\delta\) over the independent sample draw,
\[
F(\mathcal A(S))-\wh F_S(\mathcal A(S))
\le
2\sum_{i=1}^Np_i\varepsilon_i
+
C_{\rm hp}
\left[
\left(
\frac{\sum_{i=1}^N n_i\varepsilon_i^2}{n_{\rm eff}}
\right)^{1/2}
\log(2n)\log\frac{2}{\delta}
+
B_\ell
\sqrt{\frac{\log(2/\delta)}{n_{\rm eff}}}
\right].
\]
If the available stability coefficients are not bounded by \(B_\ell\), the theorem applies after replacing them by \(\min\{\varepsilon_i,B_\ell\}\).
\end{theorem}

\begin{proof}
Index samples by \(a=(i,j)\in\cI\), and write \(i(a)=i\) and
\[
w_a:=\frac{p_{i(a)}}{n_{i(a)}}.
\]
Let \(z'_a\sim\cD_{i(a)}\) be an independent ghost sample, and let \(S^a\) be \(S\) with \(z_a\) replaced by \(z'_a\). Define the deterministic function of \(S\)
\[
g_a(S)
:=
\E_{z'_a}
\left[
F_{i(a)}(\mathcal A(S^a))-\ell(\mathcal A(S^a),z_a)
\right].
\]

We verify the assumptions of Lemma~\ref{lem:weighted-bkz-moment}. Conditional on \(S_{-a}\) and \(z'_a\), the hypothesis \(\mathcal A(S^a)\) is independent of \(z_a\), while \(z_a\sim\cD_{i(a)}\). Therefore
\[
\E\!\left[
F_{i(a)}(\mathcal A(S^a))-\ell(\mathcal A(S^a),z_a)
\,\middle|\,
S_{-a},z'_a
\right]
=
0.
\]
Taking expectation over \(z'_a\) gives
\[
\E[g_a(S)\mid S_{-a}]=0.
\]
Since \(0\le \ell\le B_\ell\), also \(0\le F_i(\cdot)\le B_\ell\), and hence \(|g_a(S)|\le B_\ell\). In particular,
\[
\left|\E[g_a(S)\mid z_a]\right|\le B_\ell .
\]

Now let \(b=(r,s)\ne a\), and let \(\widetilde S\) differ from \(S\) only in coordinate \(b\). Couple the same ghost value \(z'_a\) in the definitions of \(S^a\) and \(\widetilde S^a\). Then \(S^a\) and \(\widetilde S^a\) differ in one sample of client \(r=i(b)\). Uniform stability gives, for every test point \(z\),
\[
|\ell(\mathcal A(S^a),z)-\ell(\mathcal A(\widetilde S^a),z)|
\le
\varepsilon_r.
\]
Integrating over \(z\sim\cD_{i(a)}\) gives
\[
|F_{i(a)}(\mathcal A(S^a))-F_{i(a)}(\mathcal A(\widetilde S^a))|
\le
\varepsilon_r.
\]
The loss term at \(z_a\) changes by at most \(\varepsilon_r\) as well. Therefore changing coordinate \(b\) changes \(g_a\) by at most \(2\varepsilon_{i(b)}\). Thus Lemma~\ref{lem:weighted-bkz-moment} applies with \(M=B_\ell\), weights \(w_a=p_i/n_i\), and coordinate sensitivities \(\alpha_a=2\varepsilon_{i(a)}\).
Since
\[
\sum_{a\in\cI}w_a^2
=
\sum_{i=1}^N n_i\left(\frac{p_i}{n_i}\right)^2
=
\sum_{i=1}^N\frac{p_i^2}{n_i}
=
\frac1{n_{\rm eff}},
\]
and
\[
\sum_{a\in\cI}\alpha_a^2
=
4\sum_{i=1}^N n_i\varepsilon_i^2,
\]
we have, for every \(r\ge2\),
\begin{equation}
\left\|
\sum_{a\in\cI}w_ag_a(S)
\right\|_r
\le
C
\left[
B_\ell\sqrt{\frac{r}{n_{\rm eff}}}
+
r\log(2n)
\left(
\frac{\sum_{i=1}^N n_i\varepsilon_i^2}{n_{\rm eff}}
\right)^{1/2}
\right].
\label{eq:weighted-stability-moment}
\end{equation}

It remains to compare the upper, one-sided risk gap to \(\sum_aw_ag_a(S)\). For every \(a=(i,j)\), the datasets \(S\) and \(S^a\) differ in one sample of client \(i\). Stability implies
\[
|F_i(\mathcal A(S))-F_i(\mathcal A(S^a))|\le \varepsilon_i,
\qquad
|\ell(\mathcal A(S),z_a)-\ell(\mathcal A(S^a),z_a)|\le \varepsilon_i.
\]
Therefore,
\[
F_i(\mathcal A(S))-\ell(\mathcal A(S),z_a)
\le
\E_{z'_a}\left[
F_i(\mathcal A(S^a))-\ell(\mathcal A(S^a),z_a)
\right]
+
2\varepsilon_i.
\]
Multiplying by \(w_a=p_i/n_i\) and summing over \(a\in\cI\) gives the one-sided comparison
\[
F(\mathcal A(S))-\wh F_S(\mathcal A(S))
\le
\sum_{a\in\cI}w_ag_a(S)
+
2\sum_{i=1}^Np_i\varepsilon_i.
\]
Applying Lemma~\ref{lem:moment-to-tail} to \eqref{eq:weighted-stability-moment} proves the theorem after absorbing numerical constants into \(C_{\rm hp}\).
\end{proof}

\begin{lemma}
\label{lem:weighted-mcdiarmid}
Assume \(0\le \ell\le B_\ell\), and suppose \(\mathcal A\) has stability coefficients \(\{\varepsilon_i\}\). Since losses are bounded, replace \(\varepsilon_i\) by
\[
\bar\varepsilon_i:=\min\{\varepsilon_i,B_\ell\}.
\]
Define
\[
\Phi(S):=F(\mathcal A(S))-\wh F_S(\mathcal A(S)).
\]
If \(S,S'\) are \(a=(i,j)\)-neighbors and \(w_a=p_i/n_i\), then
\[
|\Phi(S)-\Phi(S')|
\le
(2-w_a)\bar\varepsilon_i+B_\ell w_a.
\]
Consequently, for any \(\delta\in(0,1)\), with probability at least \(1-\delta\),
\[
\Phi(S)
\le
\E[\Phi(S)]
+
\sqrt{
\frac{\log(1/\delta)}{2}
\sum_{i=1}^N n_i
\left(
\left(2-\frac{p_i}{n_i}\right)\bar\varepsilon_i
+
B_\ell\frac{p_i}{n_i}
\right)^2
}.
\]
\end{lemma}

\begin{proof}
Let \(S,S'\) be \(a=(i,j)\)-neighbors and write \(w_a=p_i/n_i\). The population-risk part changes by at most \(\bar\varepsilon_i\):
\[
|F(\mathcal A(S))-F(\mathcal A(S'))|
\le
\sum_{r=1}^Np_r\E_{z\sim\cD_r}
|\ell(\mathcal A(S),z)-\ell(\mathcal A(S'),z)|
\le
\bar\varepsilon_i.
\]
For the empirical-risk part,
\begin{align*}
|\wh F_S(\mathcal A(S))-\wh F_{S'}(\mathcal A(S'))|
&\le
\sum_{b\ne a}w_b
|\ell(\mathcal A(S),z_b)-\ell(\mathcal A(S'),z_b)|
+
w_a B_\ell\\
&\le
(1-w_a)\bar\varepsilon_i+B_\ell w_a.
\end{align*}
Combining the two displays gives the bounded-difference constant
\[
c_{i,j}:=(2-w_a)\bar\varepsilon_i+B_\ell w_a.
\]
The final inequality follows from McDiarmid's bounded-differences inequality. If changing coordinate \((i,j)\) changes \(\Phi\) by at most \(c_{i,j}\), then
\[
\Prb\{\Phi-\E\Phi\ge t\}
\le
\exp\!\left(
-\frac{2t^2}{\sum_{i,j}c_{i,j}^2}
\right).
\]
Solving for \(t\) completes the proof.
\end{proof}

\begin{lemma}
\label{lem:client-concentration}
Let \(\mathcal C_s\subset[N]\), \(s=0,\ldots,R-1\), be independent uniformly sampled subsets of size \(K\). Let \(a_s\ge0\) be deterministic weights and define
\[
Y_i:=\frac1K\sum_{s=0}^{R-1}a_s\mathbf 1\{i\in\mathcal C_s\}.
\]
If \(K=N\), then
\[
Y_i=\frac1N\sum_{s=0}^{R-1}a_s
\quad\text{deterministically}.
\]
If \(K<N\), then for any \(\rho\in(0,1)\), with probability at least \(1-\rho\), simultaneously for all \(i\in[N]\),
\[
Y_i
\le
\frac1K
\left[
\frac KN\sum_{s=0}^{R-1}a_s
+
\sqrt{
2\frac KN\left(1-\frac KN\right)
\left(\sum_{s=0}^{R-1}a_s^2\right)
\log\frac{2N}{\rho}
}
+
\frac{a_{\max}}{3}\log\frac{2N}{\rho}
\right],
\]
where \(a_{\max}=\max_s a_s\).
\end{lemma}

\begin{proof}
For fixed \(i\), the indicators \(\mathbf 1\{i\in\mathcal C_s\}\) are independent across \(s\), with mean \(K/N\) and variance at most \((K/N)(1-K/N)\). The summands
\[
a_s\left(\mathbf 1\{i\in\mathcal C_s\}-K/N\right)
\]
are centered and bounded above by \(a_{\max}\). Bernstein's exponential-moment proof gives
\[
\Prb\!\left\{
\sum_{s=0}^{R-1}
a_s\left(\mathbf 1\{i\in\mathcal C_s\}-K/N\right)
\ge
\sqrt{2v x}+\frac{a_{\max}x}{3}
\right\}
\le
e^{-x},
\]
where
\[
v\le
\frac KN\left(1-\frac KN\right)\sum_{s=0}^{R-1}a_s^2.
\]
Taking \(x=\log(2N/\rho)\) and applying a union bound over \(i\in[N]\) yields the result. The full-participation case is deterministic.
\end{proof}

\section{Lipschitz-Controlled Muon Orthogonalization}
\label{app:polar}

Let \(\bM\in\R^{d_1\times d_2}\) have compact SVD \(\bM=\bU\bSigma \bV^\top\). The exact Muon direction is
\[
\msign(\bM)=\bU\bV^\top.
\]
It satisfies
\[
\ip{\bM}{\msign(\bM)}
=
\Tr(\bU\bSigma \bV^\top \bV\bU^\top)
=
\Tr(\bSigma)
=
\norm{\bM}_*,
\]
and
\[
\norm{\msign(\bM)}_F^2
=
\Tr(\bU\bV^\top \bV\bU^\top)
=
\Tr(\bU\bU^\top)
=
\rank(\bM)
\le
q.
\]
This bounded-direction property is used in FedMuon and Muon optimization proofs. Stability also requires Lipschitz sensitivity of the direction.

\begin{lemma}
\label{lem:smoothed-polar-lipschitz}
For \(\lambda>0\), define
\[
\Orth_\lambda(\bM)
:=
\bU
\operatorname{diag}\!\left(
\frac{\sigma_j}{\sqrt{\sigma_j^2+\lambda}}
\right)
\bV^\top,
\]
where \(\bM=\bU\operatorname{diag}(\sigma_j)\bV^\top\) is an SVD. Then
\[
\norm{\Orth_\lambda(\bM)-\Orth_\lambda(\bM')}_F
\le
\frac{1}{\sqrt{\lambda}}\norm{\bM-\bM'}_F
\qquad
\text{for all }\bM,\bM',
\]
and \(\norm{\Orth_\lambda(\bM)}_F\le \sqrt q\).
\end{lemma}

\begin{proof}
The map \(\Orth_\lambda\) is the gradient of the spectral function
\[
\bM\mapsto
\sum_{j=1}^q \sqrt{\sigma_j(\bM)^2+\lambda}.
\]
For the scalar function \(\phi(t)=\sqrt{t^2+\lambda}\), one has
\[
0\le \phi''(t)\le \frac1{\sqrt\lambda},
\qquad
0\le \frac{\phi'(t)}{t}\le \frac1{\sqrt\lambda}
\quad(t>0).
\]
The Hessian formula for smooth spectral functions therefore gives \(1/\sqrt\lambda\)-smoothness in Frobenius norm, equivalently \(1/\sqrt\lambda\)-Lipschitzness of the gradient. The norm bound follows because each singular value of \(\Orth_\lambda(\bM)\) is at most one.
\end{proof}

\medskip\noindent\textbf{Smoothed polar orthogonalization.}
For the smoothed-polar FedMuon specialization, run Algorithm~\ref{alg:fedmuon} with \(\Orth=\Orth_\lambda\) for a fixed \(\lambda>0\). Let \(\mathbf A_\lambda\), \(\mathbf b_\lambda\), \(a_s^\lambda\), and \(Y_i^\lambda(\mathcal C)\) be the quantities in \eqref{eq:A-matrix}--\eqref{eq:Yi-def} after replacing \(L_{\Orth}\) by \(1/\sqrt{\lambda}\). Set
\[
\bar\varepsilon_i^\lambda(\mathcal C)
:=
\min\{B_\ell,\,L_\ell\Delta_iY_i^\lambda(\mathcal C)\}.
\]
If \(\mathcal C_0,\ldots,\mathcal C_{R-1}\) are independent uniformly sampled subsets of \([N]\) of size \(K\), define \(\Psi_{R,E,K,N}^{\lambda}(\rho)\) by replacing \(a_s\) with \(a_s^\lambda\) in \eqref{eq:Psi-def}, and set
\[
\varepsilon_i^{\rm det,\lambda}(\rho)
:=
\min\bigl\{
B_\ell,\,
L_\ell\Delta_i\Psi_{R,E,K,N}^{\lambda}(\rho)
\bigr\}.
\]

\begin{corollary}
\label{cor:fedmuon-smoothed-polar-hp}
Assume Assumptions~\ref{ass:data}, \ref{ass:randomness}, \ref{ass:loss}, and \ref{ass:sampling}, and the smoothed-polar FedMuon specialization above.
Then, for every \(\delta\in(0,1)\),
\[
\Pr_S\!\left\{
F(\bX_{\rm out})-\wh F_S(\bX_{\rm out})
\le
\mathcal B_\delta\bigl(
\bar\varepsilon_1^\lambda(\mathcal C),\ldots,
\bar\varepsilon_N^\lambda(\mathcal C)
\bigr)
\right\}
\ge
1-\delta .
\]
If, in addition, Assumption~\ref{ass:grad} holds and local gradients are full-batch empirical gradients, one may take \(\Delta_i=2G_{\max}/n_i\).

If the participation schedule is uniformly random under the deterministic-amplification conditions of Corollary~\ref{cor:hp-gen-random}, then, for any \(\rho,\delta\in(0,1)\), with probability at least \(1-\rho-\delta\) over the client-participation randomness and the training sample,
\[
F(\bX_{\rm out})-\wh F_S(\bX_{\rm out})
\le
\mathcal B_\delta\bigl(
\varepsilon_1^{\rm det,\lambda}(\rho),\ldots,
\varepsilon_N^{\rm det,\lambda}(\rho)
\bigr).
\]
\end{corollary}

\begin{proof}[Proof of Corollary~\ref{cor:fedmuon-smoothed-polar-hp}]
Lemma~\ref{lem:smoothed-polar-lipschitz} shows that \(\Orth_\lambda\) is globally Frobenius-Lipschitz with constant \(1/\sqrt\lambda\). Thus Condition~\ref{ass:polar} holds with \(\mathcal T_\tau=\R^{d_1\times d_2}\) and \(L_{\Orth}=1/\sqrt\lambda\). Substituting this value into \eqref{eq:A-matrix}--\eqref{eq:Yi-def} gives \(\mathbf A_\lambda\), \(\mathbf b_\lambda\), \(a_s^\lambda\), and \(Y_i^\lambda(\mathcal C)\). Theorem~\ref{thm:hp-gen} gives the fixed-schedule bound with stability coefficients
\[
\bar\varepsilon_i^\lambda(\mathcal C)
=
\min\{B_\ell,L_\ell\Delta_iY_i^\lambda(\mathcal C)\}.
\]
In the full-batch case under Assumption~\ref{ass:grad}, Lemma~\ref{lem:fullbatch-replacement} gives
\[
\Delta_i=\frac{2G_{\max}}{n_i}.
\]

For random uniform participation, \(1/\sqrt\lambda\) is fixed before the schedule is sampled. Therefore the amplification weights \(a_s^\lambda\) are deterministic functions of the algorithmic constants, and Lemma~\ref{lem:client-concentration} applies to \(Y_i^\lambda(\mathcal C)\). Equivalently, Corollary~\ref{cor:hp-gen-random} applies with \(L_{\Orth}=1/\sqrt\lambda\), giving the stated random-participation bound.
\end{proof}

\begin{lemma}
\label{lem:ns-lipschitz}
Fix \(T\ge1\), \(\epsilon_{\rm ns}>0\), and deterministic coefficients \(\{(a_t,b_t,c_t)\}_{t=0}^{T-1}\subset\R^3\). Define
\[
\mathbf Z_0(\bM)
:=
\frac{\bM}{\norm{\bM}_F+\epsilon_{\rm ns}},
\]
and
\[
\mathbf Z_{t+1}
=
a_t\mathbf Z_t
+
b_t\mathbf Z_t\mathbf Z_t^\top\mathbf Z_t
+
c_t(\mathbf Z_t\mathbf Z_t^\top)^2\mathbf Z_t,
\qquad t=0,\ldots,T-1.
\]
Let \(\Orth_{\rm NS}^{T,\epsilon_{\rm ns}}(\bM):=\mathbf Z_T(\bM)\). Define
\[
\varrho_0=1,
\qquad
\Lambda_0=\frac{2}{\epsilon_{\rm ns}},
\]
and recursively
\[
\varrho_{t+1}
=
|a_t|\varrho_t
+
|b_t|\varrho_t^3
+
|c_t|\varrho_t^5,
\]
\[
\Lambda_{t+1}
=
\left(
|a_t|
+
3|b_t|\varrho_t^2
+
5|c_t|\varrho_t^4
\right)\Lambda_t .
\]
Then
\[
\norm{
\Orth_{\rm NS}^{T,\epsilon_{\rm ns}}(\bM)
-
\Orth_{\rm NS}^{T,\epsilon_{\rm ns}}(\bM')
}_F
\le
\Lambda_T\norm{\bM-\bM'}_F
\]
for all \(\bM,\bM'\), and
\[
\norm{\Orth_{\rm NS}^{T,\epsilon_{\rm ns}}(\bM)}_F\le\varrho_T .
\]
Thus \(\Orth_{\rm NS}^{T,\epsilon_{\rm ns}}\) satisfies Condition~\ref{ass:polar} globally with
\[
\mathcal T_\tau=\R^{d_1\times d_2},
\qquad
L_{\Orth}=\Lambda_T .
\]

Define the scalar polynomial recurrence
\[
p_0(x)=x,
\qquad
p_{t+1}(x)=a_t p_t(x)+b_t p_t(x)^3+c_t p_t(x)^5.
\]
For \(\tau\in(0,1]\), let
\[
e_{\rm pol}(T,\tau)
:=
\sup_{x\in[\tau,1]}|p_T(x)-1|.
\]
If \(\rank(\bM)=q\) and
\[
\frac{\sigma_q(\bM)}{\norm{\bM}_F+\epsilon_{\rm ns}}\ge\tau,
\]
then
\[
\norm{
\Orth_{\rm NS}^{T,\epsilon_{\rm ns}}(\bM)
-
\msign(\bM)
}_F
\le
\sqrt q\,e_{\rm pol}(T,\tau).
\]
For the classical cubic Newton--Schulz coefficients
\[
a_t=\frac32,
\qquad
b_t=-\frac12,
\qquad
c_t=0,
\]
one has
\[
e_{\rm pol}(T,\tau)\le (1-\tau^2)^{2^T}.
\]
\end{lemma}

\begin{proof}
First consider the normalization map
\[
N_{\epsilon_{\rm ns}}(\bM)=\frac{\bM}{\norm{\bM}_F+\epsilon_{\rm ns}} .
\]
For all \(\bM,\bM'\),
\begin{align*}
\norm{N_{\epsilon_{\rm ns}}(\bM)-N_{\epsilon_{\rm ns}}(\bM')}_F
&\le
\frac{\norm{\bM-\bM'}_F}{\norm{\bM}_F+\epsilon_{\rm ns}}
+
\norm{\bM'}_F
\left|
\frac{1}{\norm{\bM}_F+\epsilon_{\rm ns}}
-
\frac{1}{\norm{\bM'}_F+\epsilon_{\rm ns}}
\right|\\
&\le
\frac{\norm{\bM-\bM'}_F}{\epsilon_{\rm ns}}
+
\frac{\norm{\bM-\bM'}_F}{\epsilon_{\rm ns}}
=
\frac{2}{\epsilon_{\rm ns}}\norm{\bM-\bM'}_F .
\end{align*}
Also \(\norm{N_{\epsilon_{\rm ns}}(\bM)}_F\le1\). Hence \(\varrho_0=1\) and \(\Lambda_0=2/\epsilon_{\rm ns}\).

Let \(\norm{\mathbf X}_F,\norm{\mathbf Y}_F\le r\). By telescoping the three-factor product and using submultiplicativity,
\[
\norm{
\mathbf X\mathbf X^\top\mathbf X
-
\mathbf Y\mathbf Y^\top\mathbf Y
}_F
\le
3r^2\norm{\mathbf X-\mathbf Y}_F .
\]
Similarly, \((\mathbf X\mathbf X^\top)^2\mathbf X\) is a five-factor product
\[
\mathbf X\mathbf X^\top\mathbf X\mathbf X^\top\mathbf X .
\]
Telescoping this product against \(\mathbf Y\mathbf Y^\top\mathbf Y\mathbf Y^\top\mathbf Y\) gives five terms, each containing one factor \(\mathbf X-\mathbf Y\) or \((\mathbf X-\mathbf Y)^\top\) and four remaining factors of Frobenius norm at most \(r\). Therefore
\[
\norm{
(\mathbf X\mathbf X^\top)^2\mathbf X
-
(\mathbf Y\mathbf Y^\top)^2\mathbf Y
}_F
\le
5r^4\norm{\mathbf X-\mathbf Y}_F .
\]
Consequently, on the Frobenius ball of radius \(r\), the polynomial update
\[
\mathbf X\mapsto
a_t\mathbf X
+
b_t\mathbf X\mathbf X^\top\mathbf X
+
c_t(\mathbf X\mathbf X^\top)^2\mathbf X
\]
is Lipschitz with constant
\[
|a_t|+3|b_t|r^2+5|c_t|r^4,
\]
and maps the ball of radius \(r\) into the ball of radius
\[
|a_t|r+|b_t|r^3+|c_t|r^5 .
\]
Induction over \(t\) gives
\[
\norm{\mathbf Z_t(\bM)}_F\le\varrho_t,
\qquad
\norm{\mathbf Z_t(\bM)-\mathbf Z_t(\bM')}_F
\le
\Lambda_t\norm{\bM-\bM'}_F,
\]
for every \(t=0,\ldots,T\). Taking \(t=T\) proves the norm and Lipschitz claims. The constants \(\varrho_t\) and \(\Lambda_t\) depend only on \(T\), \(\epsilon_{\rm ns}\), and the coefficient sequence, not on the data, the participation schedule, or the trajectory. Hence Condition~\ref{ass:polar} holds globally with \(\mathcal T_\tau=\R^{d_1\times d_2}\) and \(L_{\Orth}=\Lambda_T\).

It remains to prove the approximation statement. Let
\[
\bM=\bU\operatorname{diag}(\sigma_1,\ldots,\sigma_q)\bV^\top
\]
be a compact SVD with \(\rank(\bM)=q\), and set
\[
x_j:=\frac{\sigma_j}{\norm{\bM}_F+\epsilon_{\rm ns}} .
\]
Then \(0\le x_j\le1\), and the spectral lower-bound assumption gives \(x_j\in[\tau,1]\) for all \(j\). Since
\[
\mathbf Z_0(\bM)=
\bU\operatorname{diag}(x_1,\ldots,x_q)\bV^\top,
\]
and the update uses only odd powers generated by \(\mathbf Z_t\mathbf Z_t^\top\), induction gives
\[
\mathbf Z_t(\bM)
=
\bU\operatorname{diag}(p_t(x_1),\ldots,p_t(x_q))\bV^\top .
\]
Since \(\msign(\bM)=\bU\bV^\top\) on the full-rank tube,
\begin{align*}
\norm{
\Orth_{\rm NS}^{T,\epsilon_{\rm ns}}(\bM)
-
\msign(\bM)
}_F^2
&=
\sum_{j=1}^q |p_T(x_j)-1|^2\\
&\le
q\,e_{\rm pol}(T,\tau)^2 .
\end{align*}
Taking square roots proves the coefficient-parametric approximation bound.

For the classical cubic coefficients, \(p_{t+1}(x)=p_t(x)(3-p_t(x)^2)/2\). If \(x\in[0,1]\), then \(p_t(x)\in[0,1]\) for every \(t\). Let \(s_t=p_t(x)^2\). Then
\[
s_{t+1}=\frac{s_t(3-s_t)^2}{4},
\]
and
\[
1-s_{t+1}
=
\frac{(1-s_t)^2(4-s_t)}{4}
\le
(1-s_t)^2,
\qquad 0\le s_t\le1 .
\]
If \(x\ge\tau\), then \(1-s_0\le1-\tau^2\), so
\[
1-s_T\le(1-\tau^2)^{2^T}.
\]
Since \(0\le p_T(x)\le1\),
\[
0\le1-p_T(x)\le1-p_T(x)^2=1-s_T,
\]
and therefore
\[
\sup_{x\in[\tau,1]}|p_T(x)-1|
\le
(1-\tau^2)^{2^T}.
\]
\end{proof}

\begin{proposition}
\label{prop:ns-finite-step-main}
With the definitions in Lemma~\ref{lem:ns-lipschitz}, for all \(\bM,\bM'\),
\[
\norm{
\Orth_{\rm NS}^{T,\epsilon_{\rm ns}}(\bM)
-
\Orth_{\rm NS}^{T,\epsilon_{\rm ns}}(\bM')
}_F
\le
\Lambda_T\norm{\bM-\bM'}_F,
\qquad
\norm{\Orth_{\rm NS}^{T,\epsilon_{\rm ns}}(\bM)}_F\le\varrho_T .
\]
Consequently, \(\Orth=\Orth_{\rm NS}^{T,\epsilon_{\rm ns}}\) satisfies Condition~\ref{ass:polar} globally with \(\mathcal T_\tau=\R^{d_1\times d_2}\) and \(L_{\Orth}=\Lambda_T\). If \(\rank(\bM)=q\) and \(\sigma_q(\bM)/(\norm{\bM}_F+\epsilon_{\rm ns})\ge\tau\), then
\[
\norm{
\Orth_{\rm NS}^{T,\epsilon_{\rm ns}}(\bM)
-
\msign(\bM)
}_F
\le
\sqrt q\,e_{\rm pol}(T,\tau).
\]
In the classical cubic Newton--Schulz specialization \(a_t=3/2\), \(b_t=-1/2\), and \(c_t=0\), one has \(e_{\rm pol}(T,\tau)\le (1-\tau^2)^{2^T}\).
\end{proposition}

\begin{proof}[Proof of Proposition~\ref{prop:ns-finite-step-main}]
Lemma~\ref{lem:ns-lipschitz} proves the Frobenius-Lipschitz bound, the norm bound, the global instantiation of Condition~\ref{ass:polar}, and the spectral approximation statement. The bound is data-independent because \(\Lambda_T\) is determined only by \(T\), \(\epsilon_{\rm ns}\), and the fixed coefficient sequence.
\end{proof}

\begin{remark}
If \(\epsilon_{\rm ns}=0\), the normalization map \(\bM\mapsto \bM/\norm{\bM}_F\) is not globally Lipschitz. The same proof remains valid on any tube satisfying
\[
\norm{\bM}_F,\norm{\bM'}_F\ge\nu>0,
\]
after replacing \(\Lambda_0=2/\epsilon_{\rm ns}\) by \(\Lambda_0=2/\nu\). This is then a local Lipschitz bound on that tube, not a global one; to use it in Theorem~\ref{thm:hp-gen}, the norm floor must hold for all coupled neighboring momenta at which \(\Orth\) is evaluated.
\end{remark}

For the Newton--Schulz FedMuon specialization, run Algorithm~\ref{alg:fedmuon} with
\[
\Orth=\Orth_{\rm NS}^{T,\epsilon_{\rm ns}}
\]
for fixed \(T\), \(\epsilon_{\rm ns}>0\), and fixed Newton--Schulz coefficients. Let \(\Lambda_T\) be the Lipschitz constant from Proposition~\ref{prop:ns-finite-step-main}, and let \(\mathbf A_{\rm NS}\), \(\mathbf b_{\rm NS}\), \(a_s^{\rm NS}\), and \(Y_i^{\rm NS}(\mathcal C)\) be the quantities in \eqref{eq:A-matrix}--\eqref{eq:Yi-def} after replacing \(L_{\Orth}\) by \(\Lambda_T\). For a fixed participation schedule \(\mathcal C\), set
\[
\varepsilon_i^{\rm NS}(\mathcal C)
:=
\min\{B_\ell,\,L_\ell\Delta_iY_i^{\rm NS}(\mathcal C)\}.
\]
If \(\mathcal C_0,\ldots,\mathcal C_{R-1}\) are independent uniformly sampled subsets of \([N]\) of size \(K\), and the replacement-sensitivity constants \(\Delta_i\) are fixed independently of the participation schedule, let \(\Psi_{R,E,K,N}^{\rm NS}(\rho)\) be \eqref{eq:Psi-def} computed with \(a_s^{\rm NS}\). Define
\[
\varepsilon_{i,{\rm det}}^{\rm NS}(\rho)
:=
\min\bigl\{
B_\ell,\,
L_\ell\Delta_i\Psi_{R,E,K,N}^{\rm NS}(\rho)
\bigr\}.
\]

\begin{corollary}
\label{cor:fedmuon-ns-hp}
Assume Assumptions~\ref{ass:data}, \ref{ass:randomness}, \ref{ass:loss}, and \ref{ass:sampling}, and the Newton--Schulz FedMuon specialization above.
Then, for every \(\delta\in(0,1)\),
\[
\Pr_S\!\left\{
F(\bX_{\rm out})-\wh F_S(\bX_{\rm out})
\le
\mathcal B_\delta\bigl(
\varepsilon_1^{\rm NS}(\mathcal C),\ldots,
\varepsilon_N^{\rm NS}(\mathcal C)
\bigr)
\right\}
\ge
1-\delta .
\]
If, in addition, Assumption~\ref{ass:grad} holds and local gradients are full-batch empirical gradients, one may take \(\Delta_i=2G_{\max}/n_i\).

If the participation schedule is uniformly random under the conditions used to define \(\varepsilon_{i,{\rm det}}^{\rm NS}(\rho)\), then, for any \(\rho,\delta\in(0,1)\), with probability at least \(1-\rho-\delta\) over the client-participation randomness and the training sample,
\[
F(\bX_{\rm out})-\wh F_S(\bX_{\rm out})
\le
\mathcal B_\delta\bigl(
\varepsilon_{1,{\rm det}}^{\rm NS}(\rho),\ldots,
\varepsilon_{N,{\rm det}}^{\rm NS}(\rho)
\bigr).
\]
\end{corollary}

\begin{proof}[Proof of Corollary~\ref{cor:fedmuon-ns-hp}]
By Proposition~\ref{prop:ns-finite-step-main}, the choice
\[
\Orth=\Orth_{\rm NS}^{T,\epsilon_{\rm ns}}
\]
satisfies Condition~\ref{ass:polar} globally with \(L_{\Orth}=\Lambda_T\). Substituting this value into \eqref{eq:A-matrix}--\eqref{eq:Yi-def} gives \(\mathbf A_{\rm NS}\), \(\mathbf b_{\rm NS}\), \(a_s^{\rm NS}\), and \(Y_i^{\rm NS}(\mathcal C)\). Theorem~\ref{thm:hp-gen} then gives the fixed-schedule bound with stability coefficients
\[
\varepsilon_i^{\rm NS}(\mathcal C)
=
\min\{B_\ell,L_\ell\Delta_iY_i^{\rm NS}(\mathcal C)\}.
\]
In the full-batch case under Assumption~\ref{ass:grad}, Lemma~\ref{lem:fullbatch-replacement} gives
\[
\Delta_i=\frac{2G_{\max}}{n_i}.
\]

For random uniform participation, \(\Lambda_T\) is fixed before the schedule is sampled. Therefore the amplification weights \(a_s^{\rm NS}\) are deterministic functions of the algorithmic constants, and Lemma~\ref{lem:client-concentration} applies to \(Y_i^{\rm NS}(\mathcal C)\). Equivalently, Corollary~\ref{cor:hp-gen-random} applies with \(L_{\Orth}=\Lambda_T\), giving the stated bound with
\[
\varepsilon_{i,{\rm det}}^{\rm NS}(\rho)
=
\min\bigl\{
B_\ell,\,
L_\ell\Delta_i\Psi_{R,E,K,N}^{\rm NS}(\rho)
\bigr\}.
\]
\end{proof}

In one dimension, \(\msign(m)=\operatorname{sign}(m)\) for \(m\ne0\) and \(\msign(0)=0\). For \(m=\epsilon\) and \(m'=-\epsilon\),
\[
|\msign(m)-\msign(m')|=2,
\qquad
|m-m'|=2\epsilon,
\]
so no finite global Lipschitz constant exists as \(\epsilon\downarrow0\). Matrix rank changes and near-zero singular values create the same obstruction. Condition~\ref{ass:polar} is therefore required for the stability theorem unless the orthogonalization map is smoothed or otherwise regularized.

\begin{proposition}
\label{prop:exact-sign-obstruction}
For exact scalar \(\msign\), no sample-size-decaying uniform-stability bound follows from Assumptions~\ref{ass:loss} and \ref{ass:grad} alone. More precisely, for every odd \(n\) and every \(\eta>0\), there exist neighboring full-batch datasets \(S,S'\) differing in one sample such that the one-step scalar Muon update with \(\beta=0\), \(x^0=m^0=0\), and \(\Orth=\msign\) satisfies
\[
|x^1(S)-x^1(S')|=2\eta,
\]
although there is a constant \(C>0\), independent of \(n\), such that
\[
|G_S(0)-G_{S'}(0)|\le \frac{C}{n}.
\]
Consequently, the induced loss stability can remain bounded away from zero independently of \(n\).
\end{proposition}

\begin{proof}
Let the sample space be \(\{-1,+1\}\), and choose
\[
\ell(x,z)=\frac{B_\ell}{2}+a z\tanh(x),
\qquad
0<a\le \frac{B_\ell}{2}.
\]
Then \(0\le\ell\le B_\ell\), the loss is \(a\)-Lipschitz, and
\[
\nabla_x\ell(x,z)=a z\,\operatorname{sech}^2(x)
\]
is bounded and smooth with constants independent of \(n\). Take an odd \(n\). Let \(S\) contain \((n+1)/2\) copies of \(+1\) and \((n-1)/2\) copies of \(-1\). Let \(S'\) be obtained by replacing one \(+1\) in \(S\) by \(-1\). Then
\[
G_S(0)=\frac{a}{n},
\qquad
G_{S'}(0)=-\frac{a}{n},
\qquad
|G_S(0)-G_{S'}(0)|=\frac{2a}{n}.
\]
With \(\beta=0\), the first momentum equals the full-batch gradient, so exact \(\msign\) gives
\[
x^1(S)=-\eta,
\qquad
x^1(S')=+\eta.
\]
For the test point \(z=+1\),
\[
|\ell(x^1(S),+1)-\ell(x^1(S'),+1)|
=
2a\tanh(\eta),
\]
which is independent of \(n\). Thus exact \(\msign\) has a discontinuity obstruction at zero momentum.
\end{proof}

\begin{lemma}
\label{lem:exact-sign-gap}
Let
\[
\mathcal T_\gamma
:=
\{\bM\in\R^{d_1\times d_2}:\sigma_q(\bM)\ge\gamma\},
\qquad
\gamma>0.
\]
For \(\bM\in\mathcal T_\gamma\), let \(\msign(\bM)=\bU\bV^\top\), where \(\bM=\bU\bSigma\bV^\top\) is the compact SVD with \(q\) positive singular values. Then, for all \(\bM,\bM'\in\mathcal T_\gamma\),
\[
\norm{\msign(\bM)-\msign(\bM')}_F
\le
\frac{2}{\sigma_q(\bM)+\sigma_q(\bM')}
\norm{\bM-\bM'}_F
\le
\frac1\gamma\norm{\bM-\bM'}_F .
\]
In the neighboring-tube argument below we use the more conservative constant \(2/\gamma\).
\end{lemma}

\begin{proof}
Assume first that \(d_1\ge d_2\). Then \(q=d_2\), and every matrix in \(\mathcal T_\gamma\) has full column rank. The factor \(\msign(\bM)\) is the full-column-rank polar factor of \(\bM\). Li's polar-factor perturbation theorem~\cite{li1995new} states that if \(\mathbf Z,\widetilde{\mathbf Z}\in\R^{m\times n}\), \(m\ge n\), have full column rank and polar factors \(\mathbf Q,\widetilde{\mathbf Q}\), then, for every unitarily invariant norm,
\[
\norm{\mathbf Q-\widetilde{\mathbf Q}}
\le
\frac{2}{\sigma_n(\mathbf Z)+\sigma_n(\widetilde{\mathbf Z})}
\norm{\mathbf Z-\widetilde{\mathbf Z}} .
\]
Applying this theorem with the Frobenius norm gives the sharper displayed inequality.

If \(d_1<d_2\), apply the same full-column-rank result to \(\bM^\top\) and \((\bM')^\top\). Since \(\msign(\bM^\top)=\msign(\bM)^\top\) and the Frobenius norm is invariant under transposition, the same bound follows. If \(\bM,\bM'\in\mathcal T_\gamma\), then \(\sigma_q(\bM)+\sigma_q(\bM')\ge2\gamma\), which gives the \(1/\gamma\) bound. The \(2/\gamma\) constant used below is a valid conservative relaxation.
\end{proof}

\begin{lemma}
\label{lem:gaussian-perturbation-smin}
Let \(q=\min\{d_1,d_2\}\). There is a universal numerical constant \(C_{\rm ac}>0\) such that, for every deterministic \(\mathbf Z\in\R^{d_1\times d_2}\), every \(\sigma>0\), every \(t\ge0\), and every \(\mathbf H\in\R^{d_1\times d_2}\) with independent \(N(0,1)\) entries,
\[
\Prb\!\left\{
\sigma_q(\mathbf Z+\sigma\mathbf H)\le t
\right\}
\le
\min\left\{
1,\,
C_{\rm ac}\sqrt q\,\frac{t}{\sigma}
\right\}.
\]
The same bound holds conditionally when \(\mathbf Z\) is random but independent of \(\mathbf H\).
\end{lemma}

\begin{proof}
The square case is the standard shifted Gaussian least-singular-value anti-concentration bound from smoothed analysis; see \cite{sankar2006smoothed,rudelson2008smallest,vershynin2018high}. Namely, for every deterministic \(\mathbf B\in\R^{q\times q}\), every \(\sigma>0\), and every \(\mathbf H\in\R^{q\times q}\) with independent \(N(0,1)\) entries,
\[
\Prb\!\left\{
\sigma_q(\mathbf B+\sigma\mathbf H)\le t
\right\}
\le
\min\left\{
1,\,
C_{\rm ac}\sqrt q\,\frac{t}{\sigma}
\right\}.
\]

For a rectangular matrix, assume first that \(d_1\ge d_2=q\). Let \(\mathbf P\in\R^{q\times d_1}\) select any \(q\) rows and set
\[
\mathbf C=\mathbf P\mathbf Z,
\qquad
\mathbf K=\mathbf P\mathbf H.
\]
Then \(\mathbf K\) has independent \(N(0,1)\) entries and
\[
(\mathbf Z+\sigma\mathbf H)^\top(\mathbf Z+\sigma\mathbf H)
-
(\mathbf C+\sigma\mathbf K)^\top(\mathbf C+\sigma\mathbf K)
\succeq 0.
\]
Therefore
\[
\sigma_q(\mathbf Z+\sigma\mathbf H)
\ge
\sigma_q(\mathbf C+\sigma\mathbf K),
\]
and
\[
\left\{
\sigma_q(\mathbf Z+\sigma\mathbf H)\le t
\right\}
\subseteq
\left\{
\sigma_q(\mathbf C+\sigma\mathbf K)\le t
\right\}.
\]
The case \(d_1<d_2\) follows by applying the same argument to transposes. If \(\mathbf Z\) is random but independent of \(\mathbf H\), condition on \(\mathbf Z\) and apply the deterministic bound.
\end{proof}

\begin{lemma}
\label{lem:csgn-lipschitz}
Let \(q=\min\{d_1,d_2\}\) and \(\gamma>0\). Define
\[
\phi_\gamma(s)
=
\begin{cases}
s^2/(2\gamma), & 0\le s\le\gamma,\\
s-\gamma/2, & s\ge\gamma,
\end{cases}
\qquad
\Phi_\gamma(\mathbf Z)
=
\sum_{j=1}^q \phi_\gamma(\sigma_j(\mathbf Z)),
\]
where the \(q\) singular values include zeros. Set
\[
\operatorname{csgn}_\gamma(\mathbf Z)
:=
\nabla\Phi_\gamma(\mathbf Z).
\]
Then \(\operatorname{csgn}_\gamma\) is well defined on \(\R^{d_1\times d_2}\), satisfies
\[
\norm{
\operatorname{csgn}_\gamma(\mathbf Z)
-
\operatorname{csgn}_\gamma(\mathbf W)
}_F
\le
\frac1\gamma\norm{\mathbf Z-\mathbf W}_F
\qquad
\text{for all }\mathbf Z,\mathbf W,
\]
and obeys \(\norm{\operatorname{csgn}_\gamma(\mathbf Z)}_F\le\sqrt q\). Moreover, if \(\sigma_q(\mathbf Z)\ge\gamma\), then
\[
\operatorname{csgn}_\gamma(\mathbf Z)=\msign(\mathbf Z).
\]
\end{lemma}

\begin{proof}
Let \(h(\mathbf Z)=\norm{\mathbf Z}_*\). The Moreau envelope of \(h\) with parameter \(\gamma\) is
\[
h_\gamma(\mathbf Z)
:=
\min_{\mathbf Y}
\left\{
\norm{\mathbf Y}_*
+
\frac{1}{2\gamma}\norm{\mathbf Z-\mathbf Y}_F^2
\right\}.
\]
If
\[
\mathbf Z
=
\mathbf U
\operatorname{diag}(\sigma_1(\mathbf Z),\ldots,\sigma_q(\mathbf Z))
\mathbf V^\top
\]
is a full thin SVD with \(q\) singular values, von Neumann's trace inequality reduces the proximal problem to the scalar soft-thresholding problems
\[
\min_{s_j\ge0}
\left\{
s_j+\frac{1}{2\gamma}
(\sigma_j(\mathbf Z)-s_j)^2
\right\},
\]
whose solutions are \(s_j=(\sigma_j(\mathbf Z)-\gamma)_+\). Hence
\[
\operatorname{prox}_{\gamma h}(\mathbf Z)
=
\mathbf U
\operatorname{diag}\!\left(
(\sigma_1(\mathbf Z)-\gamma)_+,\ldots,
(\sigma_q(\mathbf Z)-\gamma)_+
\right)
\mathbf V^\top
\]
and
\[
h_\gamma(\mathbf Z)
=
\sum_{j=1}^q \phi_\gamma(\sigma_j(\mathbf Z))
=
\Phi_\gamma(\mathbf Z).
\]
The Moreau envelope is differentiable and
\[
\nabla \Phi_\gamma(\mathbf Z)
=
\frac{1}{\gamma}
\left(
\mathbf Z-\operatorname{prox}_{\gamma h}(\mathbf Z)
\right)
=
\mathbf U
\operatorname{diag}\!\left(
c_1(\mathbf Z),\ldots,c_q(\mathbf Z)
\right)
\mathbf V^\top,
\qquad
c_j(\mathbf Z):=
\min\left\{\frac{\sigma_j(\mathbf Z)}{\gamma},1\right\}.
\]
This formula is independent of the chosen singular vectors, including at rank-deficient points, because it is the gradient of the globally defined convex function \(\Phi_\gamma\). Equivalently, zero singular directions have zero coefficient, and repeated positive singular values have a coefficient that is constant on the corresponding singular subspace.

It remains to prove the Lipschitz bound. Let
\[
\mathbf P=\operatorname{prox}_{\gamma h}(\mathbf Z),
\qquad
\mathbf Q=\operatorname{prox}_{\gamma h}(\mathbf W).
\]
The optimality conditions give
\[
\frac{\mathbf Z-\mathbf P}{\gamma}\in\partial h(\mathbf P),
\qquad
\frac{\mathbf W-\mathbf Q}{\gamma}\in\partial h(\mathbf Q).
\]
By monotonicity of \(\partial h\),
\[
\left\langle
(\mathbf Z-\mathbf P)-(\mathbf W-\mathbf Q),
\mathbf P-\mathbf Q
\right\rangle
\ge0.
\]
Let
\[
\mathbf R:=(\mathbf Z-\mathbf P)-(\mathbf W-\mathbf Q),
\qquad
\mathbf D:=\mathbf Z-\mathbf W.
\]
Since
\[
\mathbf D=\mathbf R+(\mathbf P-\mathbf Q),
\]
the previous inequality implies
\[
\norm{\mathbf R}_F^2
\le
\langle \mathbf R,\mathbf D\rangle
\le
\norm{\mathbf R}_F\norm{\mathbf D}_F.
\]
Thus \(\norm{\mathbf R}_F\le\norm{\mathbf D}_F\), and therefore
\[
\norm{\nabla\Phi_\gamma(\mathbf Z)-\nabla\Phi_\gamma(\mathbf W)}_F
=
\frac1\gamma\norm{\mathbf R}_F
\le
\frac1\gamma\norm{\mathbf Z-\mathbf W}_F.
\]
The norm bound follows because every singular value of \(\operatorname{csgn}_\gamma(\mathbf Z)\) is at most one. If \(\sigma_q(\mathbf Z)\ge\gamma\), then all \(q\) singular values are clipped to one, so
\[
\operatorname{csgn}_\gamma(\mathbf Z)
=
\mathbf U\mathbf V^\top
=
\msign(\mathbf Z).
\]
\end{proof}

\begin{lemma}
\label{lem:no-clipping-smoothed-sign}
Let \(\mathcal J_{\rm sign}\) be a nonempty finite ordered collection of sign calls. For each \(\tau\in\mathcal J_{\rm sign}\), let \(\mathcal F_{\tau-}\) be the sigma-field generated by the training sample, participation schedule, mini-batch randomness, and all Gaussian perturbations used before call \(\tau\). Suppose \(\mathbf H_\tau\in\R^{d_1\times d_2}\) has independent \(N(0,1)\) entries and is independent of \(\mathcal F_{\tau-}\). Let \(\bM_\tau\) be \(\mathcal F_{\tau-}\)-measurable. Then, for every \(\gamma>0\),
\[
\Prb\!\left\{
\exists \tau\in\mathcal J_{\rm sign}:
\sigma_q(\bM_\tau+\sigma\mathbf H_\tau)<\gamma
\right\}
\le
|\mathcal J_{\rm sign}|
\min\left\{
1,\,
C_{\rm ac}\sqrt q\,\frac{\gamma}{\sigma}
\right\}.
\]
Consequently, with
\[
\gamma_{\alpha,\sigma}
=
\frac{\sigma\alpha}
{C_{\rm ac}\sqrt q\,|\mathcal J_{\rm sign}|},
\qquad
\alpha\in(0,1),
\]
one has
\[
\sigma_q(\bM_\tau+\sigma\mathbf H_\tau)
\ge
\gamma_{\alpha,\sigma}
\qquad
\text{for all }\tau\in\mathcal J_{\rm sign}
\]
with probability at least \(1-\alpha\).
\end{lemma}

\begin{proof}
Fix \(\tau\). Since \(\bM_\tau\) is \(\mathcal F_{\tau-}\)-measurable and \(\mathbf H_\tau\) is independent of \(\mathcal F_{\tau-}\), Lemma~\ref{lem:gaussian-perturbation-smin} gives
\[
\Prb\!\left\{
\sigma_q(\bM_\tau+\sigma\mathbf H_\tau)<\gamma
\,\middle|\,
\mathcal F_{\tau-}
\right\}
\le
\min\left\{
1,\,
C_{\rm ac}\sqrt q\,\frac{\gamma}{\sigma}
\right\}.
\]
Taking expectations and applying a union bound over \(\tau\in\mathcal J_{\rm sign}\) proves the first claim. For the displayed choice of \(\gamma_{\alpha,\sigma}\),
\[
|\mathcal J_{\rm sign}|
C_{\rm ac}\sqrt q\,\frac{\gamma_{\alpha,\sigma}}{\sigma}
=
\alpha,
\]
which proves the second claim.
\end{proof}

\medskip\noindent\textbf{Gaussian-smoothed exact sign.}
Fix a participation schedule
\[
\mathcal C=(\mathcal C_0,\ldots,\mathcal C_{R-1}),
\qquad
|\mathcal C_r|=K,
\]
and condition on any non-Gaussian data-independent algorithmic randomness for which Assumption~\ref{ass:sampling} holds with deterministic constants \(\Delta_i\). Run FedMuon with fresh Gaussian smoothing before exact sign:
\[
\bO_i^{r,k+1}
=
\msign\!\left(
\bM_i^{r,k+1}+\sigma\mathbf H_i^{r,k+1}
\right),
\qquad
i\in\mathcal C_r,\quad k=0,\ldots,E-1,
\]
where the matrices \(\mathbf H_i^{r,k+1}\) are independent \(N(0,1)^{d_1\times d_2}\), independent of the training sample and of the past. Neighboring runs use the same Gaussian perturbations.

Let
\[
T_{\rm sign}
:=
E\sum_{r=0}^{R-1}|\mathcal C_r|
=
REK
\]
be the number of matrix-sign calls for the single matrix block considered here. For \(\alpha\in(0,1)\), define
\[
\gamma_{\alpha,\sigma}
:=
\frac{\sigma\alpha}
{C_{\rm ac}\sqrt q\,T_{\rm sign}},
\qquad
L_{\alpha,\sigma}
:=
\frac{1}{\gamma_{\alpha,\sigma}}
=
\frac{C_{\rm ac}\sqrt q\,T_{\rm sign}}{\sigma\alpha},
\]
where \(C_{\rm ac}\) is the universal constant in Lemma~\ref{lem:gaussian-perturbation-smin}. Let \(\mathbf A_{\alpha,\sigma}\), \(\mathbf b_{\alpha,\sigma}\), \(a_s^{\alpha,\sigma}\), and \(Y_i^{\alpha,\sigma}(\mathcal C)\) be the quantities in \eqref{eq:A-matrix}--\eqref{eq:Yi-def} with \(L_{\Orth}=L_{\alpha,\sigma}\). Set
\[
\bar\varepsilon_i^{\alpha,\sigma}(\mathcal C)
:=
\min\left\{
B_\ell,\,
L_\ell\Delta_iY_i^{\alpha,\sigma}(\mathcal C)
\right\}.
\]
If the participation schedule is uniformly random as in Corollary~\ref{cor:hp-gen-random}, define \(\Psi_{R,E,K,N}^{\alpha,\sigma}(\rho)\) by replacing \(a_s\) with \(a_s^{\alpha,\sigma}\) in \eqref{eq:Psi-def}, and set
\[
\varepsilon_i^{\rm det,\alpha,\sigma}(\rho)
:=
\min\left\{
B_\ell,\,
L_\ell\Delta_i
\Psi_{R,E,K,N}^{\alpha,\sigma}(\rho)
\right\}.
\]

\begin{corollary}
\label{cor:gaussian-smoothed-exact-sign-gen}
Assume Assumptions~\ref{ass:data}, \ref{ass:randomness}, \ref{ass:loss}, and \ref{ass:sampling}, and the Gaussian-smoothed exact-sign setup above.
Then, for every \(\delta\in(0,1)\), with probability at least \(1-\delta-\alpha\) over the training sample and the Gaussian smoothing matrices,
\[
F(\bX_{\rm out})-\wh F_S(\bX_{\rm out})
\le
\mathcal B_\delta\!\left(
\bar\varepsilon_1^{\alpha,\sigma}(\mathcal C),
\ldots,
\bar\varepsilon_N^{\alpha,\sigma}(\mathcal C)
\right).
\]
If, in addition, Assumption~\ref{ass:grad} holds and local gradients are full-batch empirical gradients, one may take \(\Delta_i=2G_{\max}/n_i\).

Under the deterministic-amplification conditions of Corollary~\ref{cor:hp-gen-random}, with probability at least \(1-\delta-\alpha-\rho\) over the training sample, Gaussian smoothing matrices, and random participation schedule,
\[
F(\bX_{\rm out})-\wh F_S(\bX_{\rm out})
\le
\mathcal B_\delta\!\left(
\varepsilon_1^{\rm det,\alpha,\sigma}(\rho),
\ldots,
\varepsilon_N^{\rm det,\alpha,\sigma}(\rho)
\right).
\]
\end{corollary}

\begin{proof}[Proof of Corollary~\ref{cor:gaussian-smoothed-exact-sign-gen}]
Fix the participation schedule \(\mathcal C\) and condition on the non-Gaussian algorithmic randomness as in the statement. Let
\[
\gamma_{\alpha,\sigma}
=
\frac{\sigma\alpha}
{C_{\rm ac}\sqrt q\,T_{\rm sign}},
\qquad
L_{\alpha,\sigma}
=
\frac1{\gamma_{\alpha,\sigma}}.
\]

First consider the clipped smoothed-sign algorithm driven by the same Gaussian matrices:
\[
\bO_i^{r,k+1,{\rm clip}}
=
\operatorname{csgn}_{\gamma_{\alpha,\sigma}}
\!\left(
\bM_i^{r,k+1,{\rm clip}}+\sigma\mathbf H_i^{r,k+1}
\right).
\]
After conditioning on the Gaussian matrices, the call \(\tau=(r,k,i)\) uses the deterministic map
\[
\bM
\mapsto
\operatorname{csgn}_{\gamma_{\alpha,\sigma}}
(\bM+\sigma\mathbf H_i^{r,k+1}).
\]
By Lemma~\ref{lem:csgn-lipschitz}, every such map is globally \(L_{\alpha,\sigma}\)-Lipschitz. The proof of Theorem~\ref{thm:hp-gen} and the recursion in Appendix~\ref{app:stability-recursion} only require this common per-call Lipschitz constant. Therefore Theorem~\ref{thm:hp-gen} applies to the clipped algorithm with \(L_{\Orth}=L_{\alpha,\sigma}\). Hence, after integrating over the Gaussian matrices, with probability at least \(1-\delta\) over the training sample and Gaussian matrices,
\[
F(\bX_{\rm out}^{\rm clip})-\wh F_S(\bX_{\rm out}^{\rm clip})
\le
\mathcal B_\delta\!\left(
\bar\varepsilon_1^{\alpha,\sigma}(\mathcal C),
\ldots,
\bar\varepsilon_N^{\alpha,\sigma}(\mathcal C)
\right).
\]

Next define the no-clipping event along the realized clipped trajectory:
\[
\mathcal E_{\rm nc}
:=
\left\{
\sigma_q\!\left(
\bM_i^{r,k+1,{\rm clip}}+\sigma\mathbf H_i^{r,k+1}
\right)
\ge
\gamma_{\alpha,\sigma}
\quad
\text{for all }r,\ k,\ i\in\mathcal C_r
\right\}.
\]
Order the calls lexicographically, with any fixed tie-breaking among parallel clients. At call \((r,k,i)\), the pre-sign matrix \(\bM_i^{r,k+1,{\rm clip}}\) is measurable with respect to the past before the fresh perturbation \(\mathbf H_i^{r,k+1}\). Applying Lemma~\ref{lem:no-clipping-smoothed-sign} with
\[
|\mathcal J_{\rm sign}|=T_{\rm sign}
\]
gives
\[
\Prb(\mathcal E_{\rm nc})\ge1-\alpha.
\]

On \(\mathcal E_{\rm nc}\), Lemma~\ref{lem:csgn-lipschitz} gives, at every call,
\[
\operatorname{csgn}_{\gamma_{\alpha,\sigma}}
\!\left(
\bM_i^{r,k+1,{\rm clip}}+\sigma\mathbf H_i^{r,k+1}
\right)
=
\msign\!\left(
\bM_i^{r,k+1,{\rm clip}}+\sigma\mathbf H_i^{r,k+1}
\right).
\]
Compare the clipped trajectory with the pure Gaussian-smoothed exact-sign trajectory driven by the same Gaussian matrices. The initial states are identical. If the two trajectories agree before a call, then their pre-sign momenta at that call agree. On \(\mathcal E_{\rm nc}\), the clipped and exact smoothed-sign directions at that call agree, so the next model and momentum states also agree. Induction over the ordered calls gives
\[
\bX_{\rm out}^{\rm clip}=\bX_{\rm out}
\qquad
\text{on }\mathcal E_{\rm nc}.
\]
A union bound over these two events gives the schedule-conditional claim with probability at least \(1-\delta-\alpha\).

If, in addition, Assumption~\ref{ass:grad} holds and local gradients are full-batch empirical gradients, Lemma~\ref{lem:fullbatch-replacement} gives \(\Delta_i=2G_{\max}/n_i\). For random uniform participation, intersect the event \(\mathcal E_{\rm nc}\) and the clipped-algorithm generalization event with the Bernstein participation event from Corollary~\ref{cor:hp-gen-random}. This gives success probability at least \(1-\delta-\alpha-\rho\), with \(\Psi_{R,E,K,N}^{\alpha,\sigma}(\rho)\) computed using \(a_s^{\alpha,\sigma}\).
\end{proof}

\begin{lemma}
\label{lem:exact-sign-tube-certificate}
Fix a dataset \(S\), a participation schedule \(\mathcal C\), and a realization of all data-independent algorithmic randomness. Run FedMuon on \(S\) with \(\Orth=\msign\). Assume the pathwise replacement-sensitivity condition of Assumption~\ref{ass:sampling} holds for the coupled runs considered below.

For \(L>0\), define
\[
\mathbf A_L
:=
\begin{pmatrix}
1+\eta L(1-\beta)L_g & \eta L\beta\\
(1-\beta)L_g & \beta
\end{pmatrix},
\qquad
\mathbf b_L
:=
\begin{pmatrix}
\eta L(1-\beta)\\
1-\beta
\end{pmatrix}.
\]
Let \(\mathbf e_1=(1,0)^\top\) and \(\mathbf e_2=(0,1)^\top\). For each possible replaced client \(c\in[N]\), define deterministic tube radii \(\mathbf u_c^r(L)\in\R_+^2\) and \(\mathbf v_{c,i}^{r,k}(L)\in\R_+^2\) by
\[
\mathbf u_c^0(L)=0,
\]
\[
\mathbf v_{c,i}^{r,0}(L)=\mathbf u_c^r(L),
\qquad i\in\mathcal C_r,
\]
\[
\mathbf v_{c,i}^{r,k+1}(L)
=
\mathbf A_L\mathbf v_{c,i}^{r,k}(L)
+
\mathbf 1\{i=c\}\mathbf b_L,
\qquad
k=0,\ldots,E-1,
\]
and
\[
\mathbf u_c^{r+1}(L)
=
\frac1K\sum_{i\in\mathcal C_r}\mathbf v_{c,i}^{r,E}(L).
\]
Suppose that for some \(\gamma>0\), the reference trajectory satisfies
\begin{equation}
\sigma_q(\bM_i^{r,k+1}(S))\ge 2\gamma
\qquad
\text{for all }r,\ k,\ i\in\mathcal C_r,
\label{eq:reference-gap}
\end{equation}
and the deterministic radius condition satisfies
\begin{equation}
\Delta_c\,\mathbf e_2^\top \mathbf v_{c,i}^{r,k+1}(2/\gamma)
\le
\gamma
\qquad
\text{for all }c,\ r,\ k,\ i\in\mathcal C_r.
\label{eq:tube-radius-certificate}
\end{equation}
Then for every one-sample replacement \(S'\) of \(S\) in client \(c\), coupled with the same schedule and data-independent algorithmic randomness, every neighboring momentum at which exact sign is evaluated satisfies
\[
\sigma_q(\bM_i^{\prime r,k+1}(S'))\ge\gamma
\qquad
\text{for all }r,\ k,\ i\in\mathcal C_r.
\]
Consequently, at every exact-sign evaluation along the coupled runs,
\[
\norm{
\msign(\bM_i^{r,k+1}(S))
-
\msign(\bM_i^{\prime r,k+1}(S'))
}_F
\le
\frac{2}{\gamma}
\norm{
\bM_i^{r,k+1}(S)-\bM_i^{\prime r,k+1}(S')
}_F .
\]
Moreover, the stability recursion with \(L_{\Orth}=2/\gamma\) is valid for this neighboring pair, and
\[
\norm{\bX^R(S)-\bX^R(S')}_F
\le
\Delta_c\,\mathbf e_1^\top\mathbf u_c^R(2/\gamma).
\]
\end{lemma}

\begin{proof}
Fix a neighboring dataset \(S'\) differing from \(S\) in one sample of client \(c\), and set \(L=2/\gamma\). We prove by induction over rounds \(r\) and local steps \(k\) that
\begin{equation}
\begin{pmatrix}
\norm{\bX_i^{r,k}(S)-\bX_i^{\prime r,k}(S')}_F\\
\norm{\bM_i^{r,k}(S)-\bM_i^{\prime r,k}(S')}_F
\end{pmatrix}
\le
\Delta_c\,\mathbf v_{c,i}^{r,k}(L)
\qquad
(i\in\mathcal C_r),
\label{eq:local-induction-bound}
\end{equation}
and that the server states satisfy
\begin{equation}
\begin{pmatrix}
\norm{\bX^{r}(S)-\bX^{\prime r}(S')}_F\\
\norm{\bM^{r}(S)-\bM^{\prime r}(S')}_F
\end{pmatrix}
\le
\Delta_c\,\mathbf u_c^{r}(L).
\label{eq:server-induction-bound}
\end{equation}
The base case holds because the two runs are initialized identically.

Assume \eqref{eq:server-induction-bound} holds at the beginning of round \(r\). Since each selected client initializes from the server state, \(\mathbf v_{c,i}^{r,0}(L)=\mathbf u_c^r(L)\), so \eqref{eq:local-induction-bound} holds at \(k=0\).

Now assume \eqref{eq:local-induction-bound} holds at local step \(k\) for selected client \(i\). The momentum update is performed before the sign/model update. By the pathwise gradient sensitivity assumption,
\begin{align*}
\norm{\bM_i^{r,k+1}(S)-\bM_i^{\prime r,k+1}(S')}_F
&\le
\beta
\norm{\bM_i^{r,k}(S)-\bM_i^{\prime r,k}(S')}_F\\
&\quad+
(1-\beta)L_g
\norm{\bX_i^{r,k}(S)-\bX_i^{\prime r,k}(S')}_F\\
&\quad+
(1-\beta)\Delta_c\mathbf 1\{i=c\}.
\end{align*}
Using the induction hypothesis, this gives
\[
\norm{\bM_i^{r,k+1}(S)-\bM_i^{\prime r,k+1}(S')}_F
\le
\Delta_c\,\mathbf e_2^\top\mathbf v_{c,i}^{r,k+1}(L).
\]
This step does not use any Lipschitz property of \(\msign\). By the radius condition \eqref{eq:tube-radius-certificate},
\[
\norm{\bM_i^{r,k+1}(S)-\bM_i^{\prime r,k+1}(S')}_F\le\gamma.
\]
Since the spectral norm is bounded by the Frobenius norm, Weyl's inequality and the reference gap \eqref{eq:reference-gap} imply
\[
\sigma_q(\bM_i^{\prime r,k+1}(S'))
\ge
\sigma_q(\bM_i^{r,k+1}(S))
-
\norm{\bM_i^{r,k+1}(S)-\bM_i^{\prime r,k+1}(S')}_2
\ge
2\gamma-\gamma
=
\gamma.
\]
Thus both \(\bM_i^{r,k+1}(S)\) and \(\bM_i^{\prime r,k+1}(S')\) belong to \(\mathcal T_\gamma\). Lemma~\ref{lem:exact-sign-gap} can now be applied, so the exact-sign Lipschitz bound is valid at this step:
\[
\norm{
\msign(\bM_i^{r,k+1}(S))
-
\msign(\bM_i^{\prime r,k+1}(S'))
}_F
\le
L
\norm{\bM_i^{r,k+1}(S)-\bM_i^{\prime r,k+1}(S')}_F .
\]
The model update therefore satisfies
\begin{align*}
\norm{\bX_i^{r,k+1}(S)-\bX_i^{\prime r,k+1}(S')}_F
&\le
\norm{\bX_i^{r,k}(S)-\bX_i^{\prime r,k}(S')}_F\\
&\quad+
\eta L
\norm{\bM_i^{r,k+1}(S)-\bM_i^{\prime r,k+1}(S')}_F\\
&\le
\Delta_c\,\mathbf e_1^\top\mathbf v_{c,i}^{r,k+1}(L).
\end{align*}
Together with the momentum bound, this proves \eqref{eq:local-induction-bound} at step \(k+1\).

After \(E\) local steps, the server averages the selected local states. Convexity of the Frobenius norm gives
\[
\norm{\bX^{r+1}(S)-\bX^{\prime r+1}(S')}_F
\le
\frac1K\sum_{i\in\mathcal C_r}
\norm{\bX_i^{r,E}(S)-\bX_i^{\prime r,E}(S')}_F,
\]
and the same argument applies to the momentum state. Hence
\[
\begin{pmatrix}
\norm{\bX^{r+1}(S)-\bX^{\prime r+1}(S')}_F\\
\norm{\bM^{r+1}(S)-\bM^{\prime r+1}(S')}_F
\end{pmatrix}
\le
\frac1K\sum_{i\in\mathcal C_r}
\Delta_c\,\mathbf v_{c,i}^{r,E}(L)
=
\Delta_c\,\mathbf u_c^{r+1}(L).
\]
This closes the induction over rounds. Taking the first component at \(r=R\) gives the final model perturbation bound. The proof establishes the neighboring spectral gap before each use of the exact-sign Lipschitz inequality, so the argument is non-circular.
\end{proof}

\begin{remark}
Lemma~\ref{lem:exact-sign-tube-certificate} is deterministic for a fixed reference dataset, participation schedule, and data-independent algorithmic-randomness realization. It certifies stability of that realized run against all one-sample replacements of the reference dataset. To invoke Theorem~\ref{thm:hp-gen} with exact sign, the tube condition must hold uniformly over the dataset class under consideration, with common constants \(\gamma\) and \(\Delta_i\), or on a data-independent event included in the final failure probability.
\end{remark}

\begin{remark}
The tube \(\mathcal T_\gamma\) requires \(\sigma_q(\bM)\ge\gamma\), where \(q=\min\{d_1,d_2\}\). Thus the exact-sign condition covers full-column-rank tall matrices and full-row-rank wide matrices. The stated exact-sign specialization uses the full-rank tube; fixed-rank variants would require a positive rank-\(r\) singular-value gap along all coupled neighboring trajectories.
\end{remark}

\begin{remark}
The reference-trajectory condition \(\sigma_q(\bM_i^{r,k+1}(S))\ge2\gamma\) is paired with the radius condition \eqref{eq:tube-radius-certificate}. Together they ensure that every neighboring momentum remains in \(\mathcal T_\gamma\), where the exact-sign Lipschitz constant applies in the stability recursion.
\end{remark}

\section{Full-Batch Replacement Sensitivity}
\label{app:fullbatch-replacement}

\begin{lemma}
\label{lem:fullbatch-replacement}
Assume Assumption~\ref{ass:grad}. For client \(i\), set
\[
\bG_{i,S}(\bX)=\frac1{n_i}\sum_{j=1}^{n_i}\nabla\ell(\bX,z_{i,j}).
\]
If \(S,S'\) differ only in one sample of client \(c\), then for every client \(i\) and all \(\bX,\bY\),
\[
\norm{\bG_{i,S}(\bX)-\bG_{i,S'}(\bY)}_F
\le
L_g\norm{\bX-\bY}_F+\frac{2G_{\max}}{n_c}\mathbf1\{i=c\}.
\]
\end{lemma}

\begin{proof}
If \(i\ne c\), the two empirical gradients use the same samples, and Assumption~\ref{ass:grad} gives
\[
\norm{\bG_{i,S}(\bX)-\bG_{i,S'}(\bY)}_F
\le
\frac1{n_i}\sum_{j=1}^{n_i}
\norm{\nabla\ell(\bX,z_{i,j})-\nabla\ell(\bY,z_{i,j})}_F
\le
L_g\norm{\bX-\bY}_F .
\]
For \(i=c\), let \(j_0\) be the replaced index. Then
\begin{align*}
\bG_{c,S}(\bX)-\bG_{c,S'}(\bY)
&=
\frac1{n_c}\sum_{j\ne j_0}
\left(
\nabla\ell(\bX,z_{c,j})-\nabla\ell(\bY,z_{c,j})
\right)\\
&\quad+
\frac1{n_c}
\left(
\nabla\ell(\bX,z_{c,j_0})-\nabla\ell(\bY,z'_{c,j_0})
\right).
\end{align*}
The unchanged terms contribute at most
\[
\frac{n_c-1}{n_c}L_g\norm{\bX-\bY}_F
\le
L_g\norm{\bX-\bY}_F.
\]
The changed term contributes at most \(2G_{\max}/n_c\) by the bounded-gradient part of Assumption~\ref{ass:grad}. Hence
\[
\norm{\bG_{c,S}(\bX)-\bG_{c,S'}(\bY)}_F
\le
L_g\norm{\bX-\bY}_F+\frac{2G_{\max}}{n_c}.
\]
\end{proof}

\section{Neighboring-Dataset Stability Recursion}
\label{app:stability-recursion}

Fix a client \(a_0\in[N]\) and sample \(j_0\in[n_{a_0}]\). Let \(S,S'\) be neighboring datasets differing only in \(z_{a_0,j_0}\). Run FedMuon on both datasets with the same client-participation schedule and the same remaining data-independent algorithmic randomness. Let
\[
d_X^{r}:=\norm{\bX^r-\bX^{\prime r}}_F,
\qquad
d_M^{r}:=\norm{\bM^r-\bM^{\prime r}}_F,
\qquad
\mathbf u^r:=
\begin{pmatrix}
d_X^r\\
d_M^r
\end{pmatrix}.
\]
For a selected client \(i\), define local differences
\[
d_{X,i}^{r,k}:=\norm{\bX_i^{r,k}-\bX_i^{\prime r,k}}_F,
\qquad
d_{M,i}^{r,k}:=\norm{\bM_i^{r,k}-\bM_i^{\prime r,k}}_F.
\]
At local initialization,
\[
d_{X,i}^{r,0}=d_X^r,\qquad d_{M,i}^{r,0}=d_M^r.
\]

By Assumption~\ref{ass:sampling}, for the neighboring client \(a_0\),
\[
\norm{\bG_{a_0,S}^{r,k}(\bX)-\bG_{a_0,S'}^{r,k}(\bY)}_F
\le
L_g\norm{\bX-\bY}_F+\Delta_{a_0},
\]
while for \(i\ne a_0\) the last term is absent.
Using the momentum update,
\[
d_{M,i}^{r,k+1}
\le
\beta d_{M,i}^{r,k}
+
(1-\beta)L_gd_{X,i}^{r,k}
+
(1-\beta)\Delta_{a_0}\mathbf1\{i=a_0\}.
\]
Using Condition~\ref{ass:polar} at the current orthogonalization call \(\tau=(r,k,i)\),
\[
d_{X,i}^{r,k+1}
\le
d_{X,i}^{r,k}
+
\eta L_{\Orth}d_{M,i}^{r,k+1}.
\]
Combining the last two inequalities gives the vector recursion
\[
\begin{pmatrix}
d_{X,i}^{r,k+1}\\
d_{M,i}^{r,k+1}
\end{pmatrix}
\le
\mathbf A
\begin{pmatrix}
d_{X,i}^{r,k}\\
d_{M,i}^{r,k}
\end{pmatrix}
+
\Delta_{a_0}\mathbf1\{i=a_0\}\mathbf b,
\]
where \(\mathbf A\) and \(\mathbf b\) are exactly \eqref{eq:A-matrix}--\eqref{eq:b-vector}. Iterating for \(E\) local steps,
\[
\begin{pmatrix}
d_{X,i}^{r,E}\\
d_{M,i}^{r,E}
\end{pmatrix}
\le
\mathbf A^E \mathbf u^r
+
\Delta_{a_0}\mathbf1\{i=a_0\}
\left(\sum_{k=0}^{E-1}\mathbf A^k\right)\mathbf b.
\]
Server averaging and convexity of the norm imply
\begin{align*}
\mathbf u^{r+1}
&\le
\frac1K\sum_{i\in\mathcal C_r}
\begin{pmatrix}
d_{X,i}^{r,E}\\
d_{M,i}^{r,E}
\end{pmatrix}\\
&\le
\mathbf A^E \mathbf u^r
+
\frac{\mathbf1\{a_0\in\mathcal C_r\}}{K}
\Delta_{a_0}
\left(\sum_{k=0}^{E-1}\mathbf A^k\right)\mathbf b.
\end{align*}
Since \(\mathbf u^0=0\), unrolling gives
\[
\mathbf u^R
\le
\frac{\Delta_{a_0}}{K}
\sum_{s=0}^{R-1}
\mathbf1\{a_0\in\mathcal C_s\}
\mathbf A^{E(R-1-s)}
\left(\sum_{k=0}^{E-1}\mathbf A^k\right)\mathbf b.
\]
Taking the first component yields
\[
d_X^R
\le
\Delta_{a_0}
\frac1K
\sum_{s=0}^{R-1}
a_s\mathbf1\{a_0\in\mathcal C_s\},
\]
where \(a_s\) is defined in \eqref{eq:as-def}. Equivalently, with \(Y_i(\mathcal C)\) defined in \eqref{eq:Yi-def},
\[
d_X^R
\le
\Delta_iY_i(\mathcal C)
\]
whenever \(S,S'\) differ in one sample of client \(i\). Assumption~\ref{ass:loss} then gives the stability coefficient
\[
\varepsilon_i(\mathcal C)
=
L_\ell\Delta_iY_i(\mathcal C).
\]
Since \(0\le\ell\le B_\ell\), this coefficient may be truncated to
\[
\bar\varepsilon_i(\mathcal C)
=
\min\{B_\ell,L_\ell\Delta_iY_i(\mathcal C)\}.
\]

\section{Main One-Sided Generalization Bound}
\label{app:proof-main}

\begin{proof}[Proof of Theorem~\ref{thm:hp-gen}]
Condition on a participation schedule \(\mathcal C\) and, when applicable, on a realization \(\zeta\) of the remaining data-independent algorithmic randomness such that Assumption~\ref{ass:sampling} and Condition~\ref{ass:polar} hold pathwise. Under this conditioning the call-dependent orthogonalization maps are deterministic and the algorithm is deterministic as a function of the training sample. By the stability recursion in Appendix~\ref{app:stability-recursion}, replacing one sample of client \(i\) changes the final model by at most \(\Delta_iY_i(\mathcal C)\). Assumption~\ref{ass:loss} gives loss stability
\[
\bar\varepsilon_i(\mathcal C)
=
\min\{B_\ell,L_\ell\Delta_iY_i(\mathcal C)\}.
\]
Applying Theorem~\ref{thm:weighted-sharp-stability} with failure probability \(\delta/2\) gives the branch \(\mathcal H_{\delta/2}\). Applying Lemma~\ref{lem:expected-stability} and Lemma~\ref{lem:weighted-mcdiarmid} with failure probability \(\delta/2\) gives the branch \(\mathcal G_{\delta/2}\). A union bound implies that both branches hold simultaneously with probability at least \(1-\delta\). Taking the minimum gives \(\mathcal B_\delta\), proving the claim. If the schedule and algorithmic randomness are random but independent of the training sample, integrating the conditional statement over them gives the same probability statement with the right-hand side evaluated at the realized schedule and randomness.
\end{proof}

\section{Full-Batch Specialization}
\label{app:proof-fullbatch-no-sampling}

\begin{proof}[Proof of Corollary~\ref{cor:hp-gen-fullbatch}]
For full-batch empirical gradients, Lemma~\ref{lem:fullbatch-replacement} verifies the replacement sensitivity used in the stability recursion with
\[
\Delta_i=\frac{2G_{\max}}{n_i}.
\]
Thus Appendix~\ref{app:stability-recursion} gives the schedule-conditional stability coefficients
\[
\bar\varepsilon_i^{\rm fb}(\mathcal C)
=
\min\left\{
B_\ell,\,
\frac{2G_{\max} L_\ell}{n_i}Y_i(\mathcal C)
\right\}.
\]
Applying Theorem~\ref{thm:hp-gen} with these coefficients gives Corollary~\ref{cor:hp-gen-fullbatch}.
\end{proof}

\section{Random Participation}
\label{app:proof-random-corollary}

\begin{proof}[Proof of Corollary~\ref{cor:hp-gen-random}]
Under the deterministic-amplification assumption, the weights \(a_s\) are fixed before the participation schedule is sampled. Lemma~\ref{lem:client-concentration} therefore gives, with probability at least \(1-\rho\) over the random participation schedule,
\[
Y_i(\mathcal C)\le \Psi_{R,E,K,N}(\rho)
\qquad\text{for all }i\in[N].
\]
On this event,
\[
\bar\varepsilon_i(\mathcal C)
\le
\varepsilon_i^{\rm det}(\rho)
:=
\min\{B_\ell,L_\ell\Delta_i\Psi_{R,E,K,N}(\rho)\}.
\]
Both \(\mathcal H_\delta\) and \(\mathcal G_\delta\) are coordinatewise nondecreasing in the stability coefficients, hence so is \(\mathcal B_\delta\). Therefore \(\bar\varepsilon_i(\mathcal C)\le\varepsilon_i^{\rm det}(\rho)\) may be substituted in the conditional bound.
Applying Theorem~\ref{thm:hp-gen} conditionally on the schedule with failure probability \(\delta\), and then taking a union bound with the participation event, proves \eqref{eq:main-bound-random} with probability at least \(1-\rho-\delta\). If, in addition, Assumption~\ref{ass:grad} holds and local gradients are full-batch empirical gradients, Lemma~\ref{lem:fullbatch-replacement} gives \(\Delta_i=2G_{\max}/n_i\).
\end{proof}

\section{Mini-Batch Sampling}
\label{app:minibatch}

Under Assumption~\ref{ass:grad}, for a full-batch local empirical gradient,
\[
\bG_i(\bX)=\frac1{n_i}\sum_{j=1}^{n_i}\nabla\ell(\bX,z_{i,j}),
\]
neighboring datasets differing in one sample of client \(i\) satisfy
\[
\norm{\bG_{i,S}(\bX)-\bG_{i,S'}(\bY)}_F
\le
L_g\norm{\bX-\bY}_F+\frac{2G_{\max}}{n_i}.
\]
Thus Assumption~\ref{ass:sampling} holds with \(\Delta_i=2G_{\max}/n_i\).

For a mini-batch \(\mathcal B\) of size \(b\) sampled uniformly without replacement,
\[
\bG_i(\bX;\mathcal B)
=
\frac1b\sum_{j\in\mathcal B}\nabla\ell(\bX,z_{i,j}).
\]
If \(S,S'\) differ only in \(z_{i,j_0}\), then
\[
\norm{\bG_{i,S}(\bX;\mathcal B)-\bG_{i,S'}(\bY;\mathcal B)}_F
\le
L_g\norm{\bX-\bY}_F+\frac{2G_{\max}}{b}\mathbf1\{j_0\in\mathcal B\}.
\]
Taking conditional expectation over \(\mathcal B\) gives
\[
\E_{\mathcal B}\left[\frac{2G_{\max}}{b}\mathbf1\{j_0\in\mathcal B\}\right]
=
\frac{2G_{\max}}{n_i}.
\]
Theorem~\ref{thm:hp-gen} uses pathwise sensitivity; the result below therefore conditions on realized mini-batch exposure.

For the realized-exposure statement, define, for \(s=0,\ldots,R-1\) and \(k=0,\ldots,E-1\), with \(\mathbf e_1=(1,0)^\top\),
\[
a_{s,k}
:=
\mathbf e_1^\top
\mathbf A^{E(R-1-s)+E-1-k}\mathbf b,
\qquad
a_s=\sum_{k=0}^{E-1}a_{s,k}.
\]
For a fixed schedule \(\mathcal C\) and fixed mini-batch index realization \(\mathcal B\), define
\[
Y_i^{\rm mb}(\mathcal C,\mathcal B)
:=
\max_{j\in[n_i]}
\frac1K
\sum_{s=0}^{R-1}
\mathbf1\{i\in\mathcal C_s\}
\sum_{k=0}^{E-1}
a_{s,k}\,
\frac1b\mathbf1\{j\in\mathcal B_i^{s,k}\}.
\]

\begin{theorem}[Mini-batch high-probability generalization via realized exposure]
\label{thm:minibatch-realized-exposure}
Assume Assumptions~\ref{ass:data}, \ref{ass:randomness}, \ref{ass:loss}, and \ref{ass:grad}, together with Condition~\ref{ass:polar}. Suppose each selected client uses mini-batches \(\mathcal B_i^{s,k}\subset[n_i]\) of size \(b\), sampled by data-independent index randomness and coupled identically in neighboring runs. With \(a_{s,k}\) and \(Y_i^{\rm mb}(\mathcal C,\mathcal B)\) defined above, set
\[
\varepsilon_i^{\rm mb}(\mathcal C,\mathcal B)
:=
\min\{B_\ell,\,2G_{\max} L_\ell Y_i^{\rm mb}(\mathcal C,\mathcal B)\}.
\]
Then, conditional on \((\mathcal C,\mathcal B)\), with probability at least \(1-\delta\) over the training sample,
\[
F(\bX_{\rm out})-\wh F_S(\bX_{\rm out})
\le
\mathcal B_\delta\bigl(
\varepsilon_1^{\rm mb}(\mathcal C,\mathcal B),\ldots,
\varepsilon_N^{\rm mb}(\mathcal C,\mathcal B)
\bigr).
\]
\end{theorem}

\begin{proof}
If \(S,S'\) differ only in \(z_{i,j_0}\), then for every mini-batch \(\mathcal B_i^{s,k}\),
\[
\norm{
\bG_{i,S}(\bX;\mathcal B_i^{s,k})
-
\bG_{i,S'}(\bY;\mathcal B_i^{s,k})
}_F
\le
L_g\norm{\bX-\bY}_F
+
\frac{2G_{\max}}{b}\mathbf1\{j_0\in\mathcal B_i^{s,k}\}.
\]
Repeating the neighboring-dataset recursion with the additive perturbation \((2G_{\max}/b)\mathbf1\{j_0\in\mathcal B_i^{s,k}\}\) injected at local step \(k\) of round \(s\) gives
\[
d_X^R
\le
2G_{\max}\,
\frac1K
\sum_{s=0}^{R-1}
\mathbf1\{i\in\mathcal C_s\}
\sum_{k=0}^{E-1}
a_{s,k}\frac1b
\mathbf1\{j_0\in\mathcal B_i^{s,k}\}.
\]
Taking the maximum over \(j_0\in[n_i]\) gives the stability coefficient \(2G_{\max} L_\ell Y_i^{\rm mb}(\mathcal C,\mathcal B)\), truncated by \(B_\ell\). Theorem~\ref{thm:weighted-sharp-stability} and Lemma~\ref{lem:weighted-mcdiarmid}, combined as in the proof of Theorem~\ref{thm:hp-gen}, complete the proof.
\end{proof}

\section{Proofs of Corollaries}
\label{app:corollary-proofs}

\begin{proposition}
\label{prop:amplification-radius}
Let \(c:=\eta L_{\Orth}(1-\beta)L_g\).
For \(\beta\in[0,1]\), the spectral radius of \(\mathbf A\) is
\[
\rho(\mathbf A)
=
\frac{
1+\beta+c+
\sqrt{(1-\beta+c)^2+4\beta c}
}{2}.
\]
Moreover \(\rho(\mathbf A)\le 1+\eta L_{\Orth}L_g\). If \(\eta L_{\Orth}L_g>0\) and \(\beta<1\), then generally \(\rho(\mathbf A)>1\). Horizon-independent constant amplification therefore requires an additional finite-horizon verification.
\end{proposition}

\begin{proof}
The characteristic polynomial is \(\lambda^2-(1+\beta+c)\lambda+\beta\), because \(\det(\mathbf A)=\beta\). This gives the displayed formula. Let \(\mu=\eta L_{\Orth}L_g\), so \(c=(1-\beta)\mu\). Evaluating the characteristic polynomial at \(1+\mu\) gives \(\beta\mu^2\ge0\), while the larger root is at least \(1\). It follows that \(\rho(\mathbf A)\le1+\mu\).
\end{proof}

\begin{corollary}
\label{cor:uniform-full}
Assume the conditions of Corollary~\ref{cor:hp-gen-fullbatch}. Suppose \(n_i=m\), \(p_i=1/N\), \(n=Nm\), \(K=N\), and use full-batch local gradients. Then \(Y_i(\mathcal C)=A_{R,E}/N\) and \(\bar\varepsilon_i\le (2G_{\max} L_\ell/n)A_{R,E}=:\varepsilon_{\rm full}\), where \(A_{R,E}:=\sum_{s=0}^{R-1}a_s\).
Consequently, with probability at least \(1-\delta\),
\begin{equation}
F(\bX_{\rm out})-\wh F_S(\bX_{\rm out})
\le
2\varepsilon_{\rm full}
+
C_{\rm hp}
\left[
\varepsilon_{\rm full}\log(4n)\log\frac{4}{\delta}
+
B_\ell\sqrt{\frac{\log(4/\delta)}{n}}
\right].
\label{eq:cor-uniform-full}
\end{equation}
If, for the chosen finite horizon, one verifies \(A_{R,E}\le C_{\rm amp}\eta L_{\Orth}ER\), then \eqref{eq:cor-uniform-full} becomes
\[
\widetilde{\mathcal O}\!\left(
\frac{G_{\max}L_\ell C_{\rm amp}\eta L_{\Orth}ER}{n}
+
\frac{B_\ell}{\sqrt n}
\right).
\]
This is an explicit finite-horizon amplification condition.
\end{corollary}

\begin{corollary}
\label{cor:uniform-partial}
Assume the conditions of Corollary~\ref{cor:hp-gen-random}, and assume in addition Assumption~\ref{ass:grad} with full-batch local gradients. Suppose \(n_i=m\), \(p_i=1/N\), and \(n=Nm\). Define \(\varepsilon_{\rm part}(\rho):=(2G_{\max} L_\ell/m)\Psi_{R,E,K,N}(\rho)\).
Then, with probability at least \(1-\rho-\delta\),
\[
F(\bX_{\rm out})-\wh F_S(\bX_{\rm out})
\le
2\varepsilon_{\rm part}(\rho)
+
C_{\rm hp}
\left[
\varepsilon_{\rm part}(\rho)\log(4n)\log\frac{4}{\delta}
+
B_\ell\sqrt{\frac{\log(4/\delta)}{n}}
\right].
\]
If \(a_s\le C_a\eta L_{\Orth}E\) for all \(s\), then, up to universal constants and logarithms,
\[
\Psi_{R,E,K,N}(\rho)
=
\mathcal O\!\left(
C_a\eta L_{\Orth}E
\left[
\frac{R}{N}
+
\sqrt{\frac{R(1-K/N)}{KN}\log\frac{N}{\rho}}
+
\frac{1}{K}\log\frac{N}{\rho}
\right]
\right).
\]
The displayed \(a_s\)-bound is an explicit amplification assumption; the general statement remains the one in terms of \(\Psi_{R,E,K,N}\).
\end{corollary}

\subsection{Proof of Corollary~\ref{cor:uniform-full}}

\begin{proof}
Let \(A_{R,E}:=\sum_{s=0}^{R-1}a_s\).
Under \(K=N\),
\[
Y_i(\mathcal C)
=
\frac{A_{R,E}}{N}.
\]
Since \(m=n/N\) and full-batch gradients give \(\Delta_i=2G_{\max}/m\),
\[
\bar\varepsilon_i
\le
L_\ell\frac{2G_{\max}}{m}
\frac{A_{R,E}}{N}
=
\frac{2G_{\max} L_\ell}{n}A_{R,E}
=:\varepsilon_{\rm full}.
\]
Because the coefficients are uniform, \(n_{\rm eff}=n\) and
\[
\left(
\frac{\sum_{i=1}^N n_i\bar\varepsilon_i^2}{n_{\rm eff}}
\right)^{1/2}
\le
\varepsilon_{\rm full}.
\]
The \(\mathcal H\)-branch inside \(\mathcal B_\delta\) gives
\[
F(\bX_{\rm out})-\wh F_S(\bX_{\rm out})
\le
2\varepsilon_{\rm full}
+
C_{\rm hp}
\left[
\varepsilon_{\rm full}\log(4n)\log\frac{4}{\delta}
+
B_\ell\sqrt{\frac{\log(4/\delta)}{n}}
\right].
\]
This is \eqref{eq:cor-uniform-full}.
\end{proof}

\subsection{Proof of Corollary~\ref{cor:uniform-partial}}

\begin{proof}
If \(a_s\le C_a\eta L_{\Orth}E\), then
\[
\sum_{s=0}^{R-1}a_s
\le
C_a\eta L_{\Orth}ER,
\qquad
\sum_{s=0}^{R-1}a_s^2
\le
C_a^2\eta^2 L_{\Orth}^2E^2R,
\]
and \(a_{\max}\le C_a\eta L_{\Orth}E\). Substitution into \eqref{eq:Psi-def} gives
\[
\Psi_{R,E,K,N}(\rho)
=
\mathcal O\!\left(
C_a\eta L_{\Orth}E
\left[
\frac{R}{N}
+
\sqrt{\frac{R(1-K/N)}{KN}\log\frac{N}{\rho}}
+
\frac1K\log\frac{N}{\rho}
\right]
\right),
\]
which is the displayed order. The high-probability risk bound follows from Corollary~\ref{cor:hp-gen-random} and the \(\mathcal H\)-branch with
\[
\varepsilon_{\rm part}(\rho)=\frac{2G_{\max} L_\ell}{m}\Psi_{R,E,K,N}(\rho).
\]
\end{proof}

\section{Limitations and Future Work}
\label{sec:limitations}

The scope of the result is determined primarily by the orthogonalization map and the boundedness assumptions. Deterministic exact sign is non-Lipschitz near zero singular values and rank changes. The Gaussian-smoothed exact-sign corollary replaces a deterministic trajectory gap with fresh Gaussian perturbations and an additional no-clipping failure term \(\alpha\). In the single-block union-bound argument, the Lipschitz constant scales as \(\mathcal O(REK\sqrt q/(\sigma\alpha))\) and enters the finite-horizon amplification recursion. The deterministic unperturbed exact-sign statement is formulated on a joint neighboring spectral-gap tube, or on another data-independent event included in the final failure probability. Appendix~\ref{app:polar} gives a scalar obstruction showing the need for such a gap or randomized smoothing in sample-size-decaying uniform stability.

Finite-step Newton--Schulz has a global statement when the normalization step includes a positive floor \(\epsilon_{\rm ns}>0\). Proposition~\ref{prop:ns-finite-step-main} gives a data-independent Frobenius-Lipschitz constant for any fixed finite coefficient schedule. With \(\epsilon_{\rm ns}=0\), the same argument applies on coupled neighboring momentum tubes with a positive norm floor. Approximation to the polar/sign factor is governed by the normalized singular-value lower bound and by the scalar polynomial approximation error on the corresponding spectral interval.

The high-probability theorem is stated for bounded losses. For conditionally sub-Gaussian or sub-exponential losses, the bounded-differences and weighted sharp-stability conversion would be replaced by tail-sensitive concentration. The full-batch corollaries use bounded per-example gradients to obtain the replacement sensitivity \(2G_{\max}/n_i\). A natural extension is to formulate analogous results under tail conditions better matched to modern training losses.

The theorem is finite-horizon and exact in the amplification factors \(a_s\), \(Y_i\), and \(\Psi_{R,E,K,N}\). Simplified rates require explicit amplification control. Since \(\rho(\mathbf A)\) can exceed one, horizon-independent constant amplification requires verification. A sharper analysis could exploit average stability, local curvature, or problem-dependent contraction to reduce the conservatism of these amplification factors.

The main theorem is pathwise after conditioning on data-independent randomness. Mini-batch gradients can be handled through a pathwise exposure bound, as in Appendix~\ref{app:minibatch}, or through a separate expected-stability statement over the mini-batch randomness. A tighter high-probability treatment of mini-batch exposure remains a natural extension.

\ifdefined\MAINFILE
\else
\end{document}
\fi

\end{document}